\documentclass{article}
\usepackage[round]{natbib}
\usepackage{algorithm}
\usepackage{algorithmic}
\usepackage{graphicx}
\usepackage{color}
\usepackage{mathtools}
\usepackage{amsfonts}

\usepackage[toc,page]{appendix}
\usepackage{subcaption}
\usepackage{booktabs}
\usepackage{multicol}
\usepackage{amsthm}

\newtheorem{assumption}{Assumption}
\newtheorem{claim}{Claim}

\newtheorem{theorem}{Theorem}
\newtheorem{lemma}{Lemma}

\usepackage{hyperref}

\usepackage[accepted]{icml2019}

\icmltitlerunning{Combining Parametric and Nonparametric Models for Off-Policy Evaluation}

\begin{document}

\twocolumn[
\icmltitle{Combining Parametric and Nonparametric Models for Off-Policy Evaluation}

\begin{icmlauthorlist}
\icmlauthor{Omer Gottesman}{seas}
\icmlauthor{Yao Liu}{stanford}
\icmlauthor{Scott Sussex}{seas}
\icmlauthor{Emma Brunskill}{stanford}
\icmlauthor{Finale Doshi-Velez}{seas}
\end{icmlauthorlist}

\icmlaffiliation{seas}{Harvard University}
\icmlaffiliation{stanford}{Stanford University}

\icmlcorrespondingauthor{Omer Gottesman}{gottesman@fas.harvard.edu}
\icmlkeywords{Machine Learning, ICML}

\vskip 0.3in
]
\printAffiliationsAndNotice{}

\begin{abstract}
We consider a model-based approach to perform batch off-policy evaluation in reinforcement learning. Our method takes a mixture-of-experts approach to combine parametric and non-parametric models of the environment such that the final value estimate has the least expected error. We do so by 
first estimating the local accuracy of each model and then using a planner to select which model to use at every time step as to minimize the return error estimate along entire trajectories. Across a variety of domains, our mixture-based approach outperforms the individual models alone as well as state-of-the-art importance sampling-based estimators.
\end{abstract}

\section{Introduction}

In the context of reinforcement learning (RL), off-policy evaluation (OPE) refers to the task of evaluating how good a given \emph{evaluation policy} is, using data collected under a different \emph{behavior policy}. This is in contrast with the much simpler problem of on-policy evaluation, where the behavior and evaluation policies are identical, and the value of the policy can be estimated simply by taking the average rewards accumulated over the observed trajectories. The most common set of approaches to OPE derive from importance sampling (e.g. \citet{precup2000eligibility, jiang2016doubly}); while unbiased, they tend to have prohibitively high variance unless a very large amount of data is available.  

To reduce variance, another set of approaches first use the data collected under behavior policy to learn a parametric model to approximate the environment's dynamics, and then use that model to simulate trajectories under the evaluation policy (e.g. \citet{chow2015robust}).  Unfortunately, poor model specification can lead to poor generalization and model bias even if the amount of data is infinite.  \citet{fonteneau2013batch} circumvent this issue by simulating trajectories under the evaluation policy via stitching together actual transitions observed in the data. This nonparametric approach stays closer to the data, and is thus less likely to suffer from generalization errors; unlike the parametric approach, it will be consistent in the limit of infinite data. However, in a finite batch, the required transitions to use this method may not be available in the observed data if there is a large discrepancy between the behavior and evaluation policy.

Our main contribution is to note that the parametric and nonparametric approaches above have complementary strengths: the nonparametric approach can be very accurate where data are abundant, while the parametric approach can often generalize better in situations which are not frequently observed.  We therefore propose a mixture-of-experts (MoE) approach for generating trajectories which switches between sampling transitions from a parametric and nonpaprametric model.  We treat the OPE as a planning problem: at every transition, we choose the model---parametric or nonparametric---to minimize the overall value estimate error.  We derive estimators of local errors for each model, and across a variety of domains, demonstrate that our approach produces more accurate value estimates than a myopic strategy (that does not optimize for error in the long-term), modeling using either parametric or nonparamtric approaches alone, as well as state-of-the-art importance sampling methods.

\section{Background and Notation}

We denote a Markov Decision Process (MDP) by $\langle \mathcal{X}, \mathcal{A}, \gamma, f_t, f_r, p_0, \rangle$, where $\mathcal{X}$, $\mathcal{A}$ and $\gamma$ are the state space, action space, and reward discount, respectively. For this work, we assume the state transition and rewards are deterministic functions of the current state and action such that $x_{t+1}=f_t(x_t, a_t)$ and $r_t=f_r(x_t, a_t)$. The MoE algorithm we present in this paper can be applied to the stochastic case as well, but the model error estimators we develop and use only apply to the deterministic case. $p_0(x)$ denotes the initial states distribution.

A history is a sequence $H^{(i)}\coloneqq( x_0^{(i)}, a_0^{(i)} ,r_0^{(i)} ,... ,x_T^{(i)})$ where $x_0 \sim p_0$ and the actions are chosen according to a policy $\pi$ such that $a_t \sim \pi(a_t|x_t)$.  Finally, let $\Delta(x,x')$ be a distance metric over the space $\mathcal{X}$. Throughout this paper we use the Euclidean distance as the metric over the state space, but discuss how the choice of the metric could impact the performance of the algorithm.

The value of a policy is the expected sum of discounted rewards collected by following the policy, $v^\pi \coloneqq \mathrm{E} [ g_{T} | a_t \sim \pi]$, where we defined the total return of a history as $g_{T} \coloneqq \sum_{t=0}^T \gamma^t r_t$. In off-policy evaluation, our goal is to estimate the value of an \emph{evaluation} policy, $\pi_e$, using data collected using a different \emph{behavior} policy $\pi_b$.

\section{Related Work}
\label{sec:related_work}
One common approach to OPE is to perform evaluation using importance sampling (IS) (e.g. \citet{precup2000eligibility, jiang2016doubly, thomas2016data}), where the value of the evaluation policy is estimated as a weighted average of the returns of individual trajectories, properly weighted to account for the discrepancy between the evaluation and behavior policy. This is in contrast with the nonparametric approach used in \citet{fonteneau2010model, fonteneau2013batch} where the observed data is used to simulate trajectories.

Another approach to OPE first builds parametric models of the environment given the batch data.  The value of the evaluation policy is estimated by simulating trajectories according to the built model (e.g. \citet{chow2015robust, hanna2017bootstrapping, paduraru2012off, liu2018representation}).  With this approach, care must be taken to minimize the bias of the models due to the lack of counterfactual data which may be important for predicting the dynamics under the evaluation policy \citep{johansson2016learning, shalit2016estimating, liu2018representation}.  Our approach mitigates these concerns by only using the parametric model when there exist no similar transitions in the data.

Several recent works have combined IS-based estimators and model-based estimators to produce better off-policy value estimates. The most common approach uses models as part of doubly-robust methods to reduce the variance of IS-based estimators (e.g. \citet{jiang2016doubly,thomas2016data, farajtabar2018more}). \citet{thomas2016data} go further, switching from an IS-based estimate for the initial part of a trajectory to a model-based estimate for the latter part. In contrast to their work, which only switches once from data to model, our method can switch multiple times depending on which sequence of approaches will result in the most accurate value estimate.

More broadly, the general idea of switching between data and models appears in several places in the RL optimization---rather than off-policy-evaluation---literature.  Monte-Carlo Tree Search (MCTS)~\citep{browne2012survey, coulom2006efficient} and TD(N)~\citep{watkins1989learning} evaluate policies by making several prediction steps into the future using data before switching to a model.  More recently,
\citet{doya2002multiple,parbhoo2017combining, parbhoo2018improving, peng2018improving} optimized trajectories with systems modeled by multiple experts. However, to our knowledge, these kinds of approaches have been used to optimize the value of a policy (often online) but not optimize off-policy evaluation error.

\section{Method}
\label{sec:method}

\begin{algorithm}[tb]
   \caption{MoE simulator}
   \label{alg:moe_simulator}
\begin{algorithmic}
   \STATE {\bfseries Input:} Parametric model --- $(\hat{f}_{t,p}, \hat{f}_{r,p})$; Nonparametric model --- $(\hat{f}_{t,np}, \hat{f}_{r,np})$; Initial state distribution estimate --- $\hat{p}_0$; Number of trajectories to simulate --- $N_s$; evaluation policy --- $\pi_e$.
   \\[1\baselineskip]
   \FOR{$n=1$ {\bfseries to} $N_s$}
       \STATE $x_0^{(n)} \leftarrow x_0^{(n)} \sim \hat{p}_0(x)$
       \FOR{$t=0$ {\bfseries to} $T$}
           \STATE $a_t^{(n)} \leftarrow a_t^{(n)} \sim \pi_e(a|x_t^{(n)})$
           \STATE Model $\leftarrow$ ChooseModel$(x_t, a_t)$
           \STATE $\hat{f}_{t,MoE}, \hat{f}_{r,MoE} \leftarrow$ $\hat{f}_{t,\text{Model}}, \hat{f}_{r,\text{Model}}$
           \STATE $x_{t+1}^{(n)} \leftarrow \hat{f}_{t,MoE}(x_t^{(n)}, a_t^{(n)})$
           \STATE $r_t^{(n)} \leftarrow \hat{f}_{r,MoE}(x_t^{(n)}, a_t^{(n)})$
       \ENDFOR
       \STATE $g_T^{(n)} \leftarrow \sum_{t=0}^T \gamma^t r_t$
   \ENDFOR
   \STATE {\bfseries return} $\frac{1}{N}\sum_{n=0}^{N_s} g_T^{(n)}$
\end{algorithmic}
\end{algorithm}

We now introduce our mixture-of-experts (MoE) approach for choosing between parametric and nonparametric models. Algorithm \ref{alg:moe_simulator} presents the pseudo-code for our MoE simulator. In a high-level, our approach generates $N_s$ trajectories, using the evaluation policy to provide actions and the MoE model to provide transitions, and averages their returns. Specifically, each trajectory begins by sampling an initial state from the empirical distribution in the data. For every step of simulation, we first sample an action from the evaluation policy. Next, we choose between the two models by either greedily choosing the model with the smaller estimated transition error (Algorithm \ref{alg:greedy_model_selection} in Appendix \ref{appendix:model_selection_algorithms}) or using the planning method described in Section \ref{sec:planning_method} and Algorithm \ref{alg:mcts_model_selection} in Appendix \ref{appendix:model_selection_algorithms}. We continue sampling the trajectory until some termination condition or maximum trajectory length is reached. The estimated value of the evaluation policy is given by the mean return collected over simulated trajectories.

A core contribution of this work is introducing a way to locally compare the transition prediction error for the parametric and nonparametric models. Accurate estimates are crucial to making sure that the sampled trajectories are as realistic as possible, which is essential to accurately estimate the value of the evaluation policy. In Sections \ref{sec:estimating_errors_non_parametric} and \ref{sec:estimating_errors_parametric} we introduce and motivate these error estimators, and in Appendix \ref{sec:error_estimators_empirical_evaluation} we empirically evaluate the quality of these estimators and their ability to locally select the more accurate of the two models.

\subsection{Planning to optimize the policy value error bound}
\label{sec:planning_method}

We first motivate why we might wish to use a planner to minimize the value estimation error, rather than simply choosing the model that is most accurate for the current state-action pair.  Imagine a situation where one of our models is very accurate for a transition in the current state-action pair, but the next state lies in a region where both our models perform very poorly. In such a situation, we may be willing to accept some error in estimating the immediate transition, in order to continue simulating trajectories in regions of the state space where we have high confidence in our models. 

Below, we quantify this trade-off by computing the bound on the error for the reward collected over an entire trajectory, and use a planning algorithm to select the model at each time step to minimize that bound. The derivation of the bound is closely related to the derivations in \citet{asadi2018lipschitz}, with the distinction that we assume the magnitude of model error changes across different regions in the state action space, and we consider the case of deterministic transition and reward functions rather than stochastic ones. Furthermore, we consider modeling errors of both transition and reward functions, rather than only the transition function.

We first bound the state estimation error at a given time step, $\delta(t) \coloneqq \Delta(x_t, \hat{x}_t)$ where  $\hat{x}_t$ is the state at time $t$ given that the entire trajectory was simulated using the MoE model (i.e --- $\hat{x}_t=\hat{f}_t(\hat{x}_{t-1})=\hat{f}_t(\hat{f}_t(\hat{x}_{t-2}))=...$).

\begin{lemma}
\label{lemma:state_error}
Let $\varepsilon_t(t)$ be the transition estimation error bound for the chosen model at time-step $t$,
\begin{equation}
\varepsilon_t(t) \geq \Delta(\hat{x}_{t+1}, f_t(\hat{x}_{t}, a_{t}))
\end{equation}
The state error at time-step $t$ is: 
\begin{equation}
    \delta(t) \coloneqq \Delta(x_t, \hat{x}_t) \leq \sum_{t'=0}^{t-1} (L_t)^{t'} \varepsilon_t(t-t'-1)
\end{equation}
where $L_t$ is the Lipschitz constant of the transition function, $f_t$.
\end{lemma}

\emph{Proof.} The proof of Lemma \ref{lemma:state_error} is presented in Appendix \ref{sec:proof_of_state_error_lemma}.

Next we compute the bound on the total return error for a particular trajectory.

\begin{theorem}
\label{theorem:total_return_error}
Let $\varepsilon_r(t)$ be the reward estimation error bound for the chosen model at time-step $t$
\begin{equation}
\varepsilon_r(t) \geq |f_r(\hat{x}_t, a_t) - \hat{f}_r(\hat{x}_t, a_t)|
\end{equation}
The total return error for a trajectory is bounded by:
\begin{align}
    \label{eq:return_error_bound}
    \delta_g & \coloneqq |g_T - \hat{g}_T| \\ \nonumber
    & \leq \sum_{t=0}^T \gamma^t [ L_r \delta(t) + \varepsilon_r(t)] \\ \nonumber
    & \leq L_r \sum_{t=0}^T \gamma^t \sum_{t'=0}^{t-1} (L_t)^{t'} \varepsilon_t(t-t'-1) + \sum_{t=0}^T \gamma^t \varepsilon_r(t),
\end{align}
where $L_r$ is the Lipschitz constant of the reward function, $f_r$.
\end{theorem}

\emph{Proof.} From the triangle inequality we have that the bound on the reward for a given time step is

\begin{align}
    |r_t - \hat{r}_t| & \leq |f_r(x_t, a_t) - f_r(\hat{x}_t, a_t)| \\ \nonumber
    &+ |\hat{f}_r(\hat{x}_t, a_t) - f_r(\hat{x}_t, a_t)| \\ \nonumber
    & \leq L_r \delta(t) + \varepsilon_r(t).
\end{align}

The proof is completed by summing over all time steps and substituting in Lemma \ref{lemma:state_error}. We note that for $L_t>1$ the return error bound grows exponentially with the planning horizon, reflecting the compounding error phenomenon which is common to planning with imperfect models.

In order to simulate trajectories which minimize this bound we use a Monte-Carlo Tree Search algorithm (MCTS) \citep{coulom2006efficient, browne2012survey} to select the model to simulate the next transition with at each time step. Pseudo-code for the MCTS implemetation of the model selection algorithm is presented in Algorithm \ref{alg:mcts_model_selection} in Appendix \ref{appendix:model_selection_algorithms}. We emphasize that the domain over which the MCTS algorithm plans is \emph{not} the same domain the RL agent operates on. For the MCTS algorithm, a ``state'' is a state-action pair from the RL domain, possible ``actions'' are a choice between the parametric and nonparametric model, and the return is the right hand side of the bound above.

In practice, the one-step transition and reward errors, $\varepsilon_t$ and $\varepsilon_r$ in Equation \ref{eq:return_error_bound}, are unknown. In the next sections we will introduce methods for estimating upper bounds for $\varepsilon_t$ and $\varepsilon_r$ for both the parametric and nonparametric model.

\subsection{Estimating the nonparametric model error}
\label{sec:estimating_errors_non_parametric}

The nonparametric model chooses transitions from transitions that have been actually observed in the data.  Specifically, given a state $x$ and action $a \sim \pi_e(a|x)$ the nonparametric model predicts as the next state and reward the corresponding features for the transition whose starting state, $x^*_t$, is closest to $x$, and whose action, $a^*_t$, equals $a$.

The transition prediction error for the nonparametric model is thus
\begin{equation}
\label{eq:non_parametric_error}
\varepsilon_{\text{t,np}} = \Delta(f_t(x, a), x^*_{t+1}) = \Delta(f_t(x, a), f_t(x^*_t, a)).
\end{equation}
which can be bounded using the Lipschitz constant of the transition function $L_t$:
\begin{equation}
\label{eq:non_parametric_error_bound}
\varepsilon_{\text{t,np}} \leq L_t \Delta(x,x^*_t).
\end{equation}

The Lipschitz constant $L_t$ can be estimated by computing
\begin{equation}
\label{eq:Lipschitz_est}
\hat{L}_t = \max_{i \neq j} \frac{\Delta(x^{(i)}_{t'+1},x^{(j)}_{t''+1})}{\Delta(x^{(i)}_{t'},x^{(j)}_{t''})}.
\end{equation}
over all pairs of states $x^{(i)}, x^{(j)}$ in the data \citep{wood1996estimation}.

However, in practice we expect that this estimate will be too conservative, as we wish to estimate the error locally and for a specific action. For a more realistic estimate of the nonparametric error, we can use Equation (\ref{eq:Lipschitz_est}), but only consider transition pairs from within a given neighborhood of radius $C$.  (We will use the same radius for estimating the parametric error in Section \ref{sec:estimating_errors_parametric}; we will describe how to choose $C$ in Section~\ref{sec:choosing_c}.)  Our final estimate will then be 
\begin{equation}
\label{eq:non_parametric_err_est}
\varepsilon_{\text{t,np}} \approx \hat{L_t} \Delta(x,x^*_t).
\end{equation}
where $\hat{L_t}$ is estimated by using transitions starting within $C$ of $x$ in Equation \ref{eq:Lipschitz_est}.

We can similarly estimate the reward error as

\begin{equation}
\label{eq:non_parametric_reward_err_est}
\varepsilon_{\text{r,np}} \approx \hat{L_r} \Delta(x,x^*_t)
\end{equation}

\begin{equation}
\label{eq:reward_Lipschitz_const}
\hat{L}_r = \max_{i \neq j} \frac{|r^{(i)}_{t'} - r^{(j)}_{t''} |}{\Delta(x^{(i)}_{t'},x^{(j)}_{t''})}.
\end{equation}

\subsection{Estimating the parametric model error}
\label{sec:estimating_errors_parametric}

The parametric model error we wish to estimate is
\begin{equation}
\label{eq:parametric_error}
\varepsilon_{\text{t,p}} = \Delta(f_t(x, a), \hat{f}_t(x, a)),
\end{equation}
where $\hat{f}_t(x, a)$ is the parametric model prediction for the next state given $(x, a)$. We estimate the value of $\varepsilon_{\text{p}}$ as the maximum error over all transitions whose initial state is within distance $C$ of the state of interest $x$ and whose action is $a$:
\begin{equation}
\label{eq:parametric_err_est}
\hat{\varepsilon}_{\text{t,p}} = \max \Delta \left( \hat{f}_t(x_{t'}^{(i)},a), x_{t'+1}^{(i)} \right).
\end{equation}

We use the same neighborhood radius for estimating $\hat{\varepsilon}_{\text{p}}$ as we do for estimating $\hat{L}$ for the non-parametric error bound introduced in Section \ref{sec:estimating_errors_non_parametric}.

Similarly we can estimate the parametric reward error as
\begin{equation}
\label{eq:parametric_reward_err_est}
\hat{\varepsilon}_{\text{r,p}} = \max |\hat{f}_r(x_{t'}^{(i)},a) - r_{t'}^{(i)} |.
\end{equation}

\subsection{Choosing $C$}
\label{sec:choosing_c}
It remains to choose the radius $C$.  A large choice of $C$ may be too conservative, smoothing out variation in the data, whereas a small $C$ will result in high variance estimates as there will be few pairs near the desired point $x$.  Here, we discuss one natural choice for specifying the distance radius $C$: Let $\hat{\varepsilon}_{\text{p}}^g$ and $\hat{L}_t^g$ be the global average parametric dynamics model error and Lipschitz constant for the transition function respectively, computed using all transitions in the data. Then $C \hat{L}_t^g$ is an estimate of the nonparametric model error if the closest point is at distance $C$. We therefore set our radius $C$ to be when this nonparametric model error equals the global average parametric model error:
\begin{eqnarray}
\label{eq:R_value}
C\hat{L_t}^g = \hat{\varepsilon}_{\text{p}}^g \Rightarrow
C = \frac{\hat{\varepsilon}_{\text{p}}^g}{\hat{L_t}^g},
\end{eqnarray}
This choice states that for any distance greater than $C$, the nonparametric error estimate will exceed the global average model error.  

Finally, for any choice of $C$ that is smaller than the diameter of the data set, we may encounter situations in which there exist no observed transitions within the defined neighborhood $C$ of the current state $x$.  Then there will be no pairs available to estimate the Lipschitz constant $\hat{L}_t$, and no way to accurately estimate the nonparametric error.  In this situation, we default to the parametric model assuming that it will likely extrapolate more gracefully than the nonparametric model.

\subsection{Consistency of the MCTS-MoE simulator}
Our primary contribution is to develop an estimator that 
provides improved empirical performance in limited data settings 
(which often necessitates model based off policy evaluation approaches). 
However consistency, the property of converging asymptotically to the 
correct true value, is a highly desirable property for an estimator 
and provides a nice reassurance of fundamental soundness. Our MoE algorithm for estimating the value $v^{\pi_e}$ of the desired evaluation policy $\pi_e$ is consistent under some assumptions.

\begin{theorem} 
\label{theorem:consistency}
Under the assumptions \ref{assm:behavior_converage}, \ref{assm:C_converage}, \ref{assm:parametric_model_Lipschitz} in Appendix \ref{sec:consistency}, assuming planning error $\epsilon_{\text{planning}} = o(1)$, the MoE simulator with MCTS model selection is a consistent estimator of policy value of $\pi_e$.

\end{theorem}
Due to space limitations, the proof and details on the assumptions are 
deferred to the Appendix \ref{sec:consistency}. Furthermore, in Appendix \ref{sec:consistency_alternate_model_selection} we provide proof that the assumption of planning error going to zero is not required for consistency if the model error for the reward is also included in the greedy MoE model selection.

\section{Demonstrations in Synthetic Domains}
\label{sec:toy_examples}

\subsection{Motivating example for mixtures}
\label{sec:2d_gridworld}
\begin{figure}[t]
\centering
\subcaptionbox{\label{fig:simple_example_domain}}
{\includegraphics[width=0.15\textwidth]{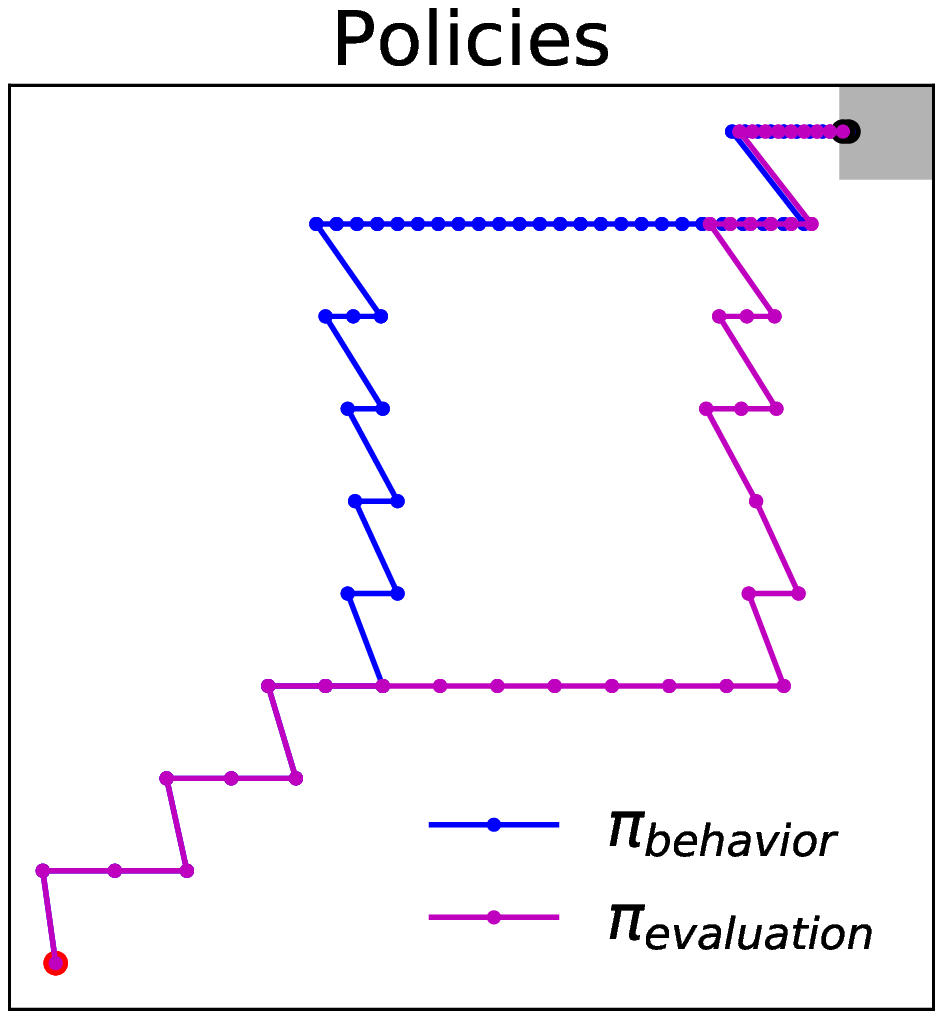}}
\subcaptionbox{\label{fig:simple_example_true_wind}}
{\includegraphics[width=0.15\textwidth]{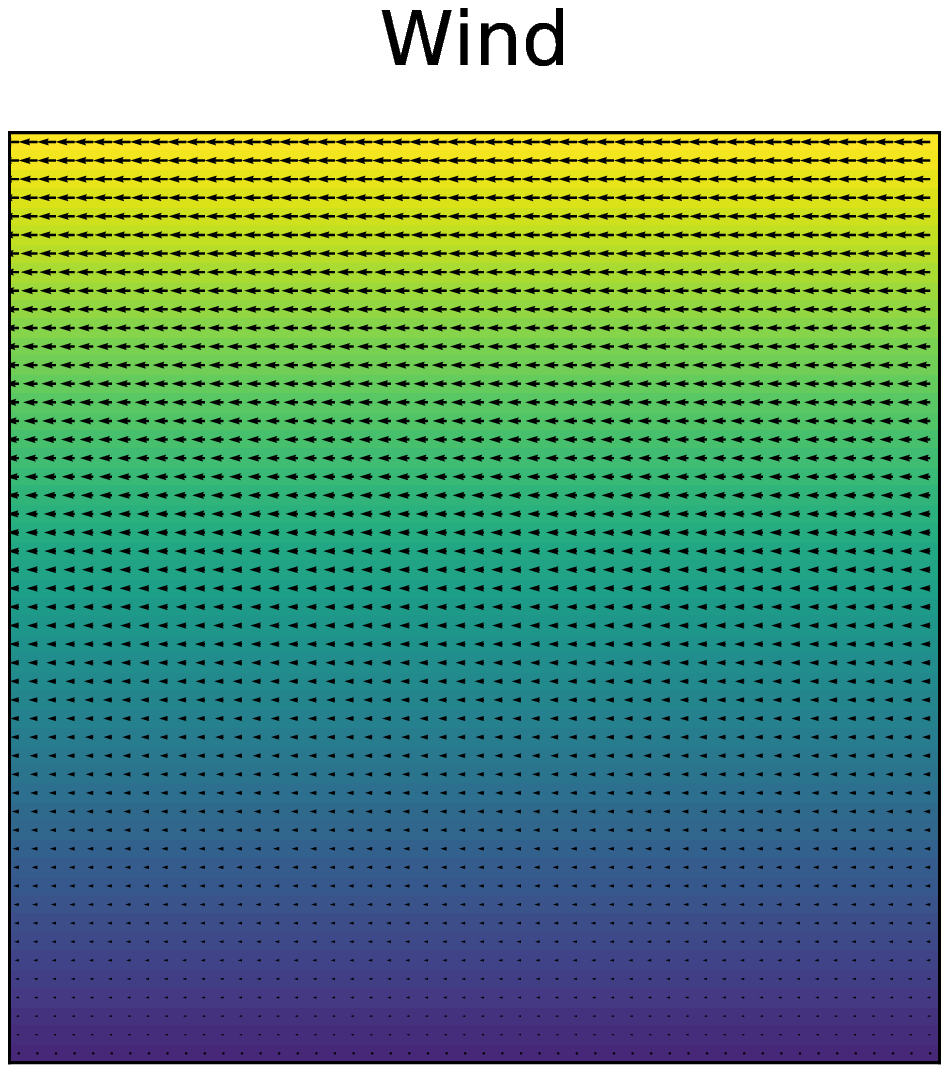}}
\subcaptionbox{\label{fig:simple_example_parametric_model}}
{\includegraphics[width=0.15\textwidth]{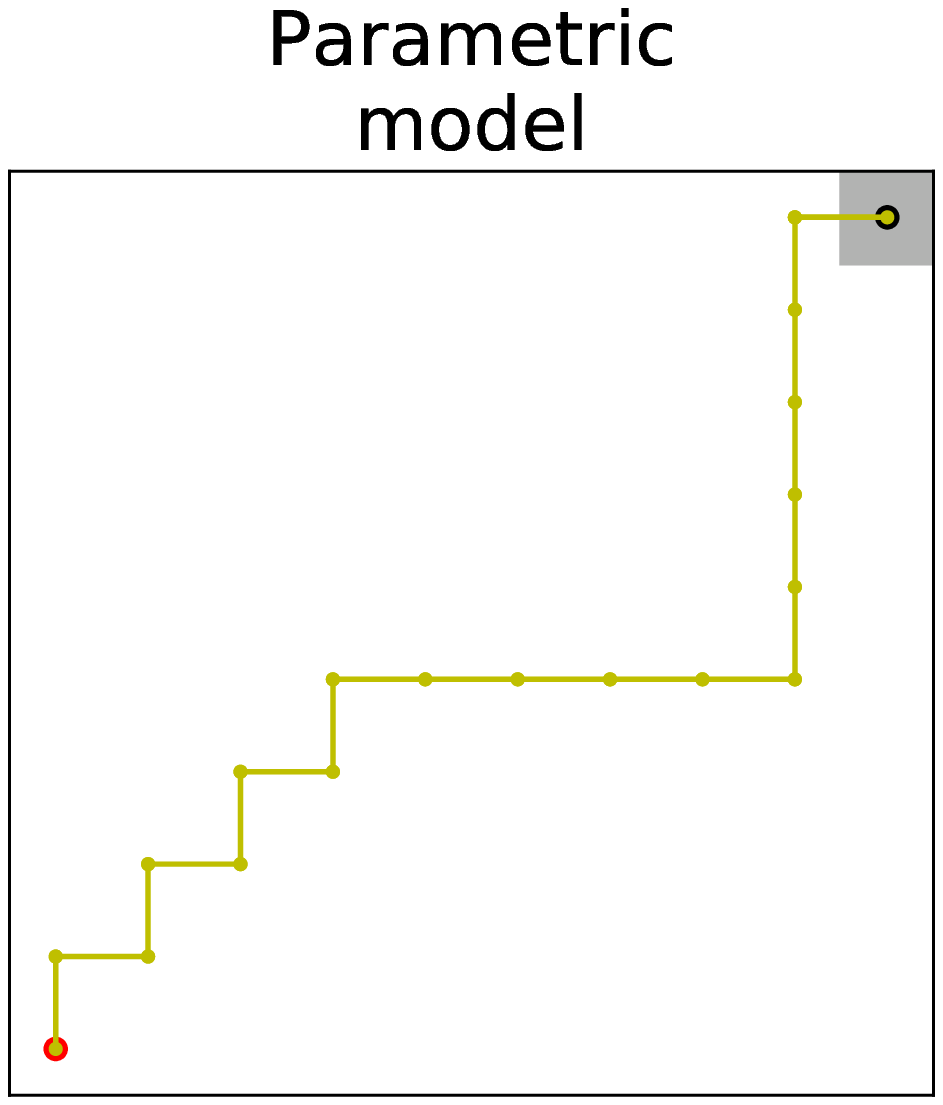}}
\subcaptionbox{\label{fig:simple_example_non_parametric_model}}
{\includegraphics[width=0.15\textwidth]{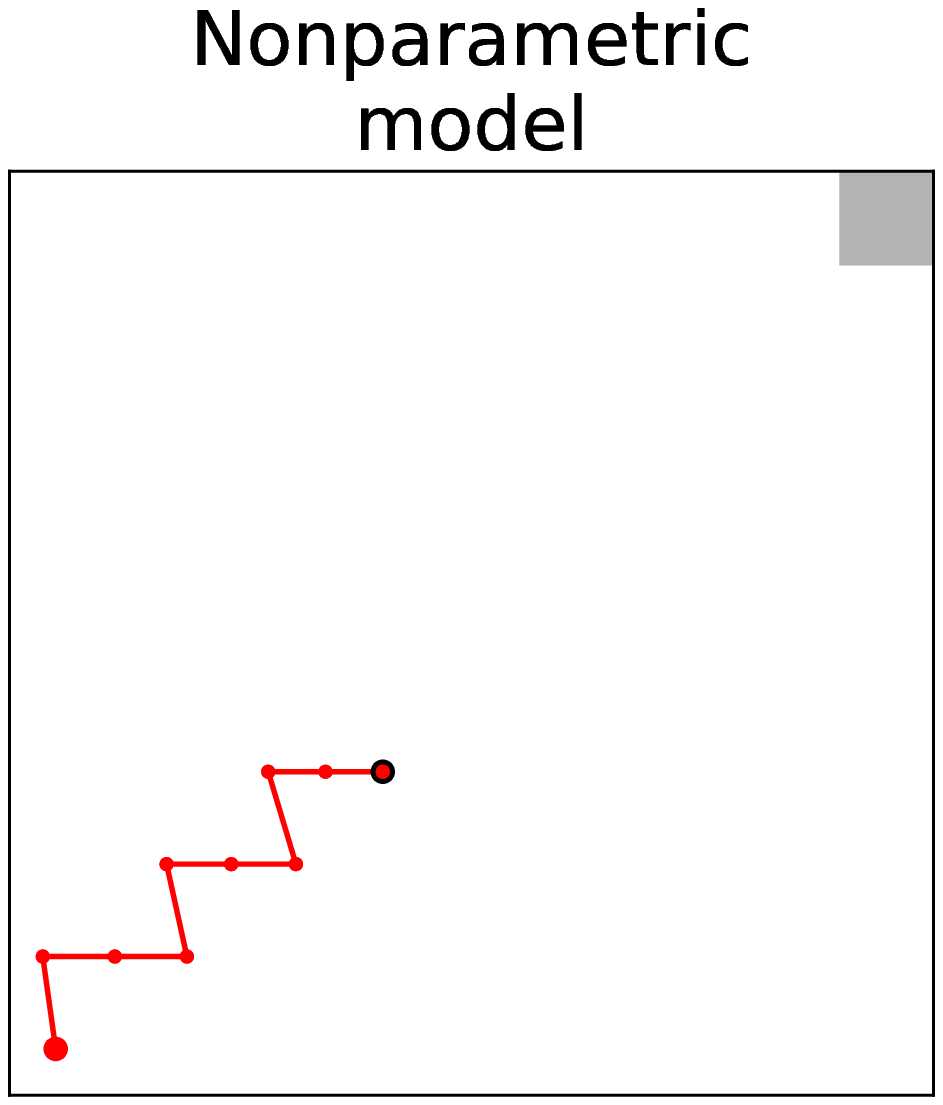}}
\subcaptionbox{\label{fig:simple_example_moe_model}}
{\includegraphics[width=0.15\textwidth]{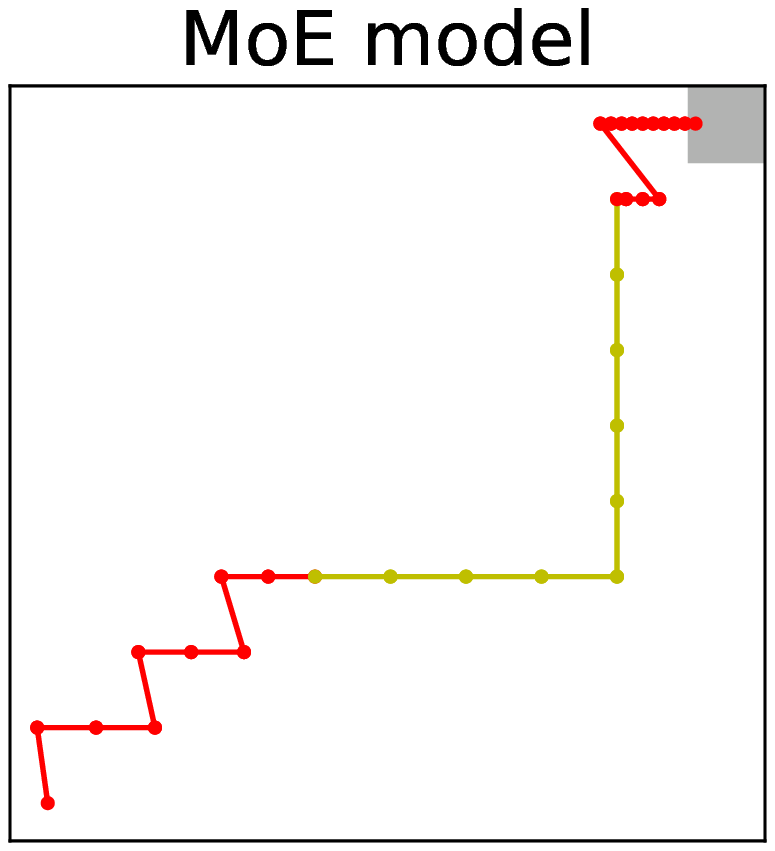}}
\caption{\textbf{Demonstration in a continuous 2D map.} The MoE model switches from using the non-parametric to the parametric model in regions where no transitions are observed.}\label{fig:simple_example}
\end{figure}

\begin{table}
\vspace{-0.5em}
\caption{Value estimates in 2D motivating example}
\label{table:simple_example}
\small
\centering
\begin{tabular}{cc|ccc}
\toprule
\multicolumn{1}{c}{$v^{\pi_e}$} & 
\multicolumn{1}{c|}{$v^{\pi_b}$} &
\multicolumn{1}{c}{$\hat{v}_{M_{\text{p}}}^{\pi_e}$} &
\multicolumn{1}{c}{$\hat{v}_{M_{\text{np}}}^{\pi_e}$} &
\multicolumn{1}{c}{$\hat{v}_{M_{\text{MoE}}}^{\pi_e}$} \\
\midrule
-40 &-53 &-18 &-$\infty$ &\textbf{-31} \\
\bottomrule
\end{tabular}
\end{table}

In this section we illustrate the advantages of of using mixtures in a very simple setting of a myopic planner.  We consider a two-dimensional navigation domain. The agent's state is represented by a $x \in \mathbb{R}^2$ coordinate in space. The action space is discrete with 4 actions: up, down, left, right. The transition to the next state is given by
\begin{equation}
x_{t+1} = x_t + \Delta_{ss} \cdot a_t + w(x_t)
\end{equation}
where $\Delta_{ss}$ is a constant which determines the step size, $a_t$ is the chosen action represented as a unit vector in $\mathbb{R}^2$, and $w(x_t)$ is a state dependent ``wind" that pushes the agent away from the expected direction.

The agent follows a policy starting from its initial state until it reaches a goal region. The reward is $r = -1$ for all non-goal states and the discount factor is 1. Thus, the value of a policy is minus the expected number of steps required to reach the goal region.  In Figure \ref{fig:simple_example_domain} we show trajectories in which the agent starts at the bottom left of the domain and attempts to reach the gray area at the top right. The wind increases linearly with the y-coordinate and is directed in the negative direction of the x-coordinate (Figure \ref{fig:simple_example_true_wind}). Figure \ref{fig:simple_example_domain} shows trajectories generated under the behavior and evaluation policy. Because the wind is stronger near the top of the domain, the value of the evaluation policy is higher than the behavior's (Table \ref{table:simple_example}).

\paragraph{Advantages of each individual model.} As baselines we consider the performance of the parametric and nonparametric models separately. The parametric model has access to the direction and step size of each action, but does not model the wind. Because the parametric model does not take into account the wind which slows the agent down, it overestimates $\hat{v}_{M_{\text{p}}}^{\pi_e}$ (Table \ref{table:simple_example}), as can be observed in Figure \ref{fig:simple_example_parametric_model}. The non-parametric model does a better job of generating a realistic trajectory in the lower left region of the space where the evaluation and behavior policies are similar, but is unable to simulate a trajectory which continues past the region where the policies deviate from each other (Figure \ref{fig:simple_example_non_parametric_model}). Because the nonparametric model is unable to generate trajectories which reach the goal state, it estimates $\hat{v}_{M_{\text{p}}}^{\pi_e}$ to be $-\infty$.

\paragraph{Combining the strengths of both models using the MoE model.} The MoE model manages to utilize the best of both models (Figure \ref{fig:simple_example_moe_model}). It generates a realistic trajectory using the nonparametric model where the behavior and evaluation policies match, and switches to using the parametric model where they don't and observed transitions are not available. Note that the MoE model switches back to using the nonparametric model near the goal where the number of observed transitions is once again high. As shown in Table \ref{table:simple_example}, the MoE model generates the closest estimate for $v^{\pi_e}$. Note that even for the MoE model, the estimator cannot converge to the true value of the policy, as it must use the overly optimistic parametric model where no data is collected.

\subsection{Motivating examples for Planning with Mixtures}

\begin{figure*}
\centering
\subcaptionbox{\label{fig:planning_toy_1}}
{\includegraphics[width=0.32\textwidth]{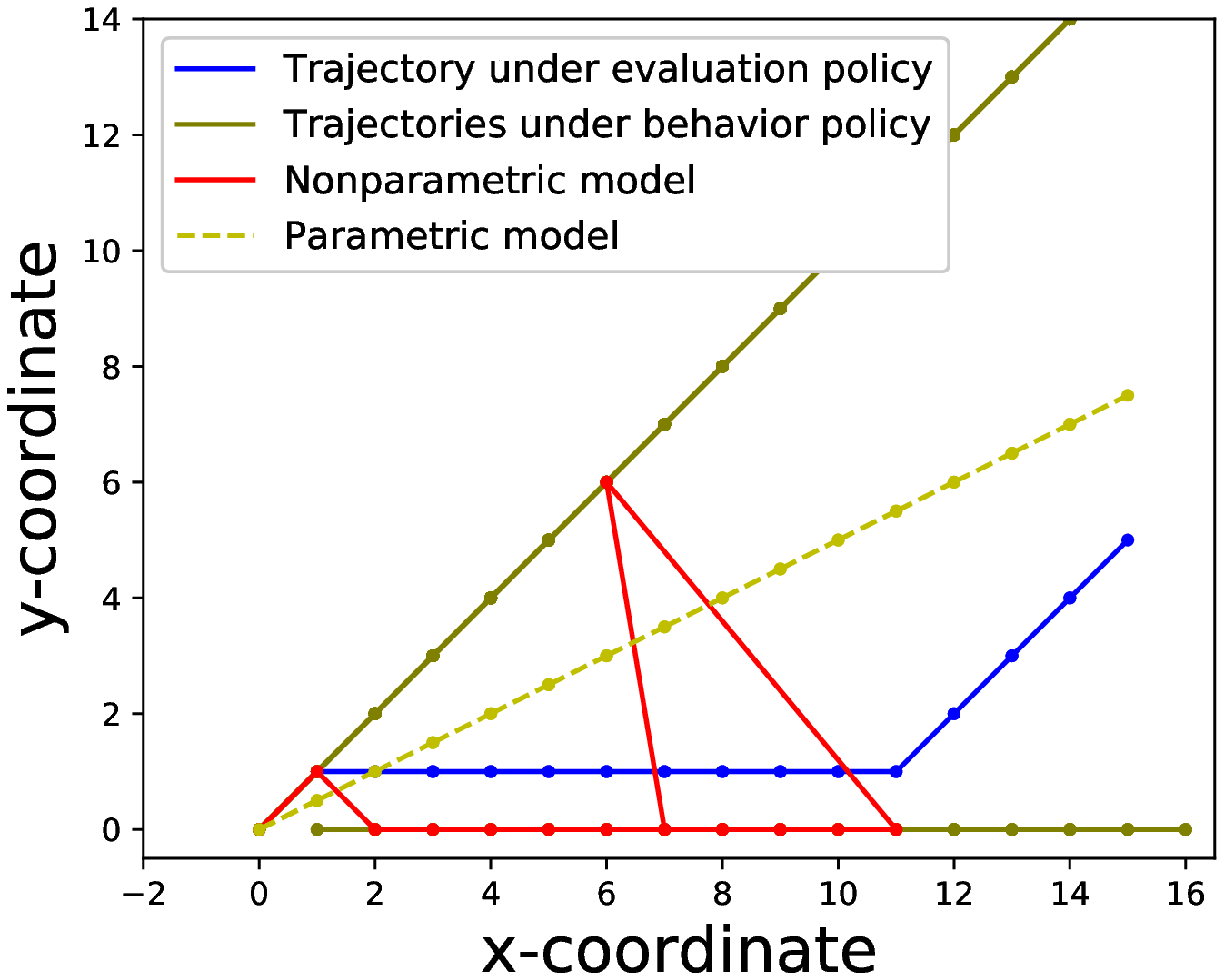}}
\subcaptionbox{\label{fig:planning_toy_2}}
{\includegraphics[width=0.32\textwidth]{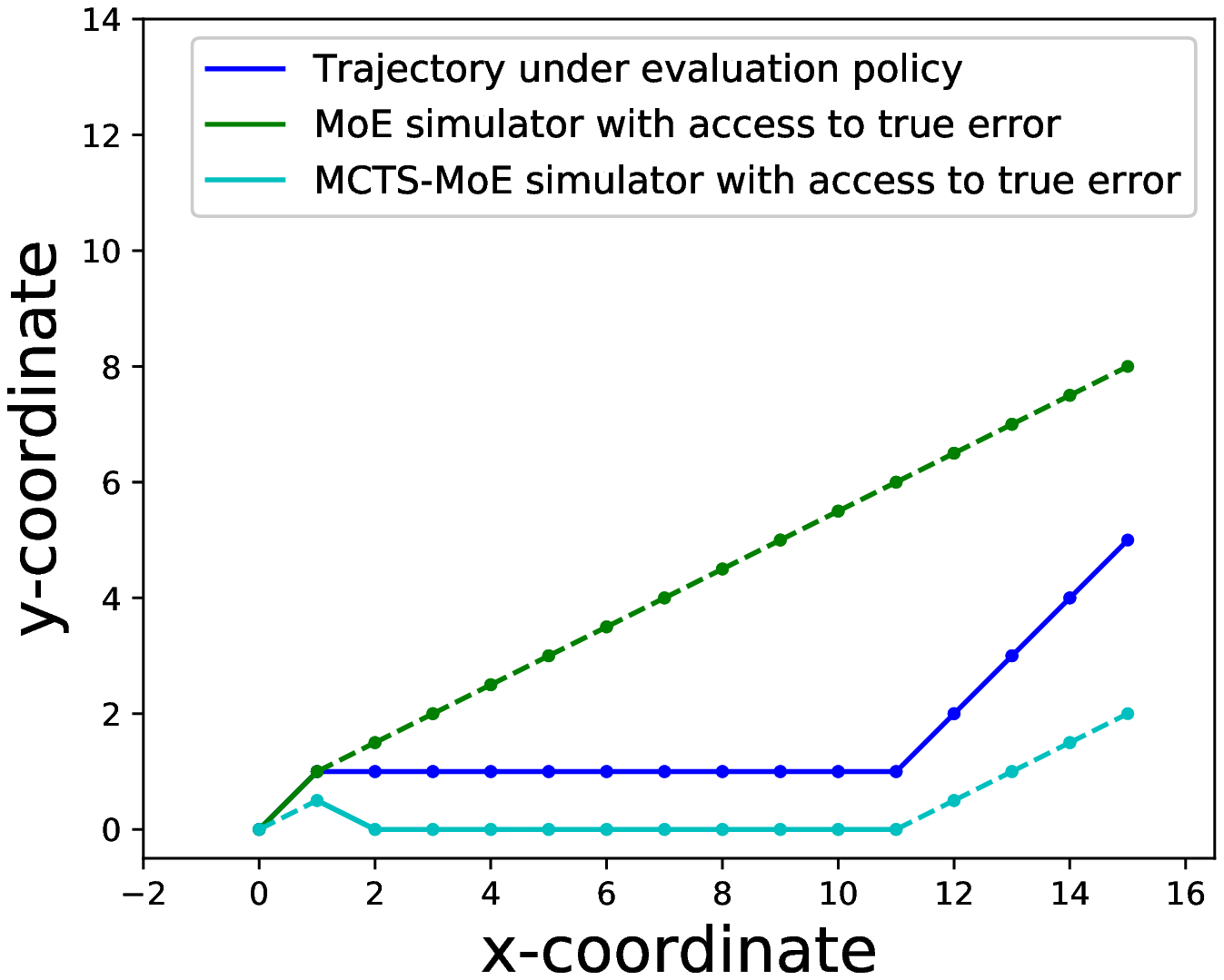}}
\subcaptionbox{\label{fig:planning_toy_3}}
{\includegraphics[width=0.32\textwidth]{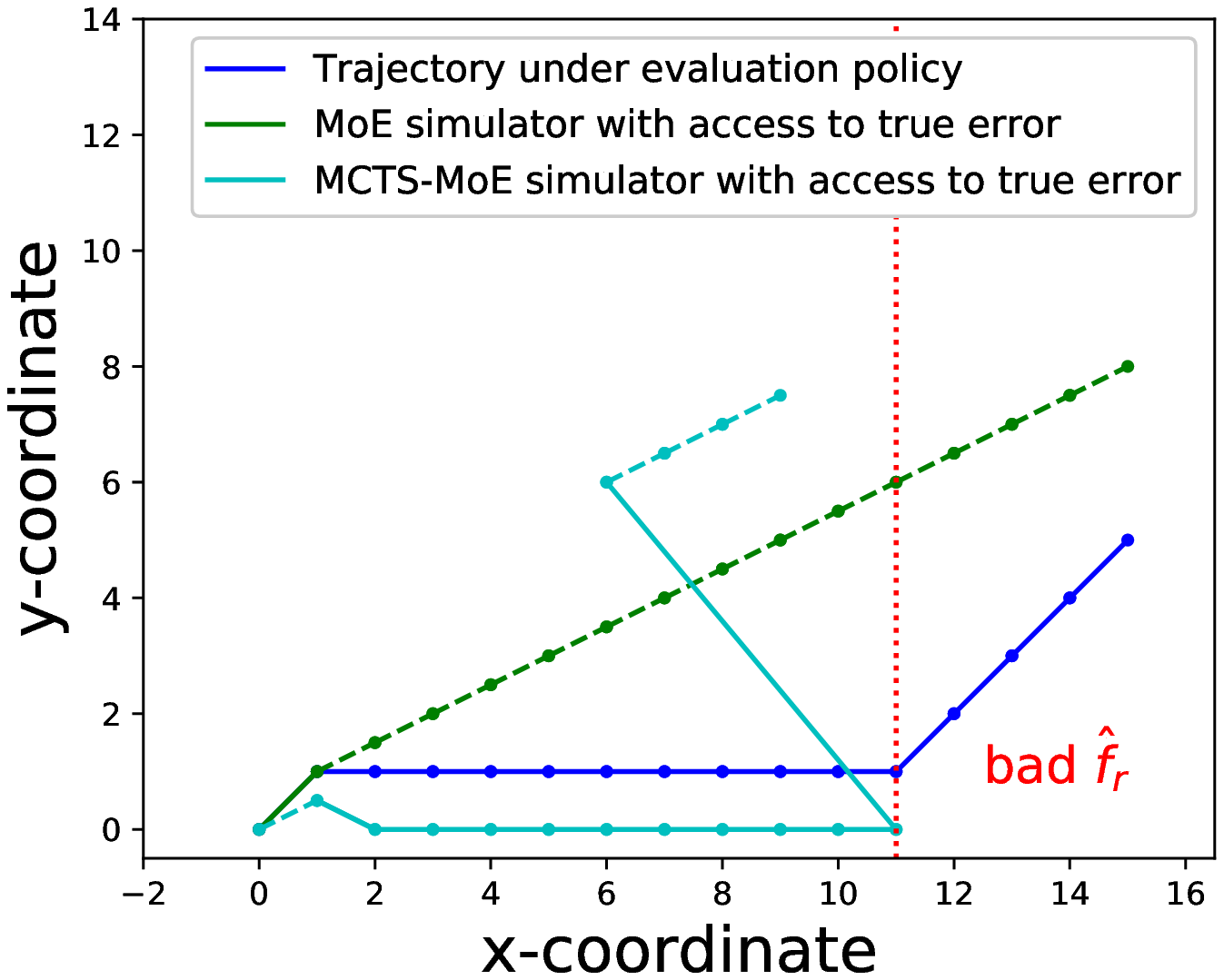}}
\caption{\textbf{Demonstration of planning in a 2D navigation domain.} By using planning to minimize the return estimation error over entire trajectories the MCTS-MoE simulator can incur immediate transition error to decrease long term trajectory simulation error (b) and avoid simulating trajectories where the reward estimation error is too high (c).}
\label{fig:planning_toy}
\end{figure*}

The previous example demonstrated how the MoE simulator, even with a myopic policy, can capitalize on the complementary strengths of parametric and nonparametric models. In this section we demonstrate how further accuracy can be gained by performing the model selection using planning to minimize the errors accumulated over entire trajectories.

This domain is also a 2D navigational domain where the state is defined as $x=(x_{(1)}, x_{(2)})$. There are two possible actions --- ``right (r)'' for which $f_t(x,a=``r")=(x_{(1)}+1, x_{(2)})$, and 
``diagonal (d)'' for which $f_t(x,a=``d")=(x_{(1)}+1, x_{(2)}+1)$. The evaluation and behavior policies are given by

\begin{equation}
\pi_e(a|x)=
\begin{cases}
``r" & 1 \leq x_{(1)} \leq 11 \\
``d" & \text{otherwise}
\end{cases}
\end{equation}

\begin{equation}
\pi_b(a|x)=
\begin{cases}
``r" & x_{(1)} > 0 \text{ and } x_{(2)} = 0 \\
``d" & \text{otherwise}
\end{cases}
\end{equation}

Initial states for collecting observed trajectories under the behavior policy are $(0,0)$ and $(1,0)$, and we wish to evaluate the value of the evaluation policy given the initial state $x_0=(0,0)$. The trajectories under the behavior and evaluation policies are shown in Figure \ref{fig:planning_toy_1}. Also shown in Figure \ref{fig:planning_toy_1} are the trajectories generated using the nonparametric and parametric models, where the parametric model for the transitions predicts $\hat{f}_t(x,a)= f_t(x)= (x_{(1)}+1,x_{(2)}+0.5)$ regardless of the action (i.e. a model which is not accurate enough to distinguish between actions, but predicts the average of the transitions for each action).

\paragraph{Planning with a perfect reward model.} The reward for this domain is $f_r(x,a)= f_r(x) = x_{(1)} + x_{(2)}$. In Figure \ref{fig:planning_toy_2} and \ref{fig:planning_toy_3} we demonstrate the trajectories simulated by the basic MoE simulator and MCTS-MoE simulator, where solid and dashed lines represent transitions simulated using the nonparametric and parametric models, respectively. In the example presented in Figure \ref{fig:planning_toy_2}, the parametric model predicts the true reward at a given state $(f_r=\hat{f}_r)$. Because there is no error associated with predicting the reward, the only error in evaluating the evaluation policy stems from errors in predicting the states visited. In this section, we give the MCTS-MoE model access to the true errors of both the parametric and nonparametric reward as our goal is to demonstrate the effect of planning on value estimation, rather than investigate the quality of our error estimators, as we do in Appendix \ref{sec:error_estimators_empirical_evaluation}.

At the first time-step the basic MoE simulator uses the nonparametric model to simulate a transition with zero error, and then uses the parametric model to simulate the following transitions as they incur smaller transition errors than the nonparametric model. However, while the error at each time-step is smaller using the parametric error, by greedily choosing the model which minimizes the immediate transition error the MoE simulator generates a trajectory which over time becomes very different from the true trajectory under the evaluation policy. In contrast, the MCTS-MoE simulator can look into the future and realize that incurring a small transition error in the first time step will lead it to the bottom region of the state space where it can generate transitions with no error by using the nonparametric model. In the first row in Table \ref{table:planning_toy_results} we present the value estimation error for each of the models and see that the MCTS-MoE outperforms both parametric and nonparametric models, as well as the basic MoE simulator. It is worth mentioning that in this case the basic MoE model performs even worse than the individual models, since greedily choosing the more accurate model in the short term leads it to generate very unrealistic trajectories in the long run --- the planning aspect of the MCTS-MoE model is designed to avoid exactly this problem.

\paragraph{Balancing reward and transition errors.} So far we considered the case in which the parametric model for the reward is accurate, and therefore minimizing the value estimation error is equivalent to minimizing the states estimation error. Next, we consider the case where the parametric model for the reward may be inaccurate as well. In Figure \ref{fig:planning_toy_3} we plot the MCTS-MoE simulated trajectory where $\hat{f}_{r,p}=f_r$ for $x<11$ and $\hat{f}_{r,p}=-1$ for $x \geq 11$. In this situation the reward estimation error for simulating a transition using the parametric model is so high for $x \geq 11$, that the MCTS-MoE simulator prefers to incur a large transition estimation error, which is balanced by avoiding the large reward estimation error. In the second row of Table \ref{table:planning_toy_results} we show that in this case as well the MCTS-MoE simulator outperforms all other simulators.

\begin{table}
\vspace{-0.5em}
\caption{Value estimation errors in the planning toy example}
\label{table:planning_toy_results}
\small
\centering
\begin{tabular}{c|cccc}
\toprule
\multicolumn{1}{c}{} & 
\multicolumn{1}{c}{$M_{\text{p}}$} &
\multicolumn{1}{c}{$M_{\text{np}}$} &
\multicolumn{1}{c}{$M_{\text{MoE}}$} &
\multicolumn{1}{c}{$M_{\text{MCTS-MoE}}$} \\
\midrule
Accurate $f_{r,p}$ & 32.5 & 39 & 39.5 &\textbf{17} \\
Inaccurate $f_{r,p}$ & 46.5 & 39 & 41.5 &\textbf{19.5} \\
\bottomrule
\end{tabular}
\end{table}

\section{Experimental Results}
\subsection{Conceptual demonstration --- Acrobot}

\paragraph{Domain details.} As previously discussed, our MoE simulator will offer the biggest advantages in situations where some regions in state space which are expected to be visited by the evaluation policy are not observed under the behavior policy. In these cases the parametric model can be used to generalize and simulate the dynamics in the unobserved regions, while the nonparametric model can offer more accurate predictions in the data rich regions.

To demonstrate this property we use the Acrobot environment from the control literature \citep{sutton1996generalization}. The environment simulates two links and two joints, where the joint between the two links can be actuated. The objective of a policy is to control the actuation of the middle joint such that the end of the bottom joint rises above a certain height as quickly as possible. The reward for every time step is $-1$, and so the value of a policy is the average time it takes to reach the goal height. We train a near-optimal policy with an expectation time of 70 time-steps for the completion of the task. We generate 100 observed trajectories, where trajectories differ from each other due to small perturbations of the initial states. To generate a lack of observed transitions in a particular region in space, we run several experiments, and in each experiment we choose a maximal observable height, and remove from the dataset all observed transitions which start above that height.

We then train a parametric and nonparametric model on the data, and compare the performance of the two models and our greedy MoE model in predicting the expected time it would take for the desired height to be reached. The parametric model is trained as a feed-forward neural net with one layer of 64 hidden units with a $\tanh$ activation function.

\paragraph{Results}
In Figure \ref{fig:acrobot_results} we present the RMSE of $\hat{v}^{\pi_e}$ for the different models as a function of the maximal observable height. For low maximal observable heights, the nonparametric model cannot simulate a trajectory in which the Acrobot reaches the desired height, as such transitions are not observed in the data. For such a situation the MoE model fully relies on the parametric model and matches its performance. As the maximal height is increased, the non-parametric model becomes more viable, and the MoE model combines transition predictions from both models and outperforms both individual models. For very high maximal observable heights (the goal height is 1), the nonparametric model becomes very accurate and outperforms the MoE model. This is likely due to errors in our estimate of the transition prediction error, which leads the MoE to select the parametric model in situations where the non-parametric would be more accurate.

\begin{figure}
\centering
\includegraphics[width=0.3\textwidth]{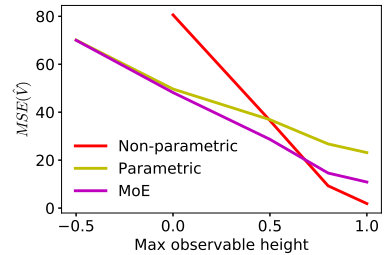}
\caption{\textbf{Acrobot.} The MoE model relies on the parametric models in reigmes when there is not sufficient data for the non-parametric model to be reliable, and combines the advantages of both models in regimes where the models complement each other.}
\label{fig:acrobot_results}
\end{figure}

\subsection{Medical simulators}
\label{sec:results_medical_simulators}

We compare our MoE simulator with different OPE estimators for two medical simulators: one for cancer \citep{ribba2012tumor} and one for HIV \citep{ernst2006clinical} patients. In Appendix \ref{sec:experimental_details} we provide details of the simulators and the evaluation policies used. For both domains, we use as behavior policies $\epsilon$-greedy policies of the evaluation policy.

\begin{figure}[t]
\centering
\subcaptionbox{Cancer\label{fig:cancer_traj_err}}
{\includegraphics[width=0.23\textwidth]{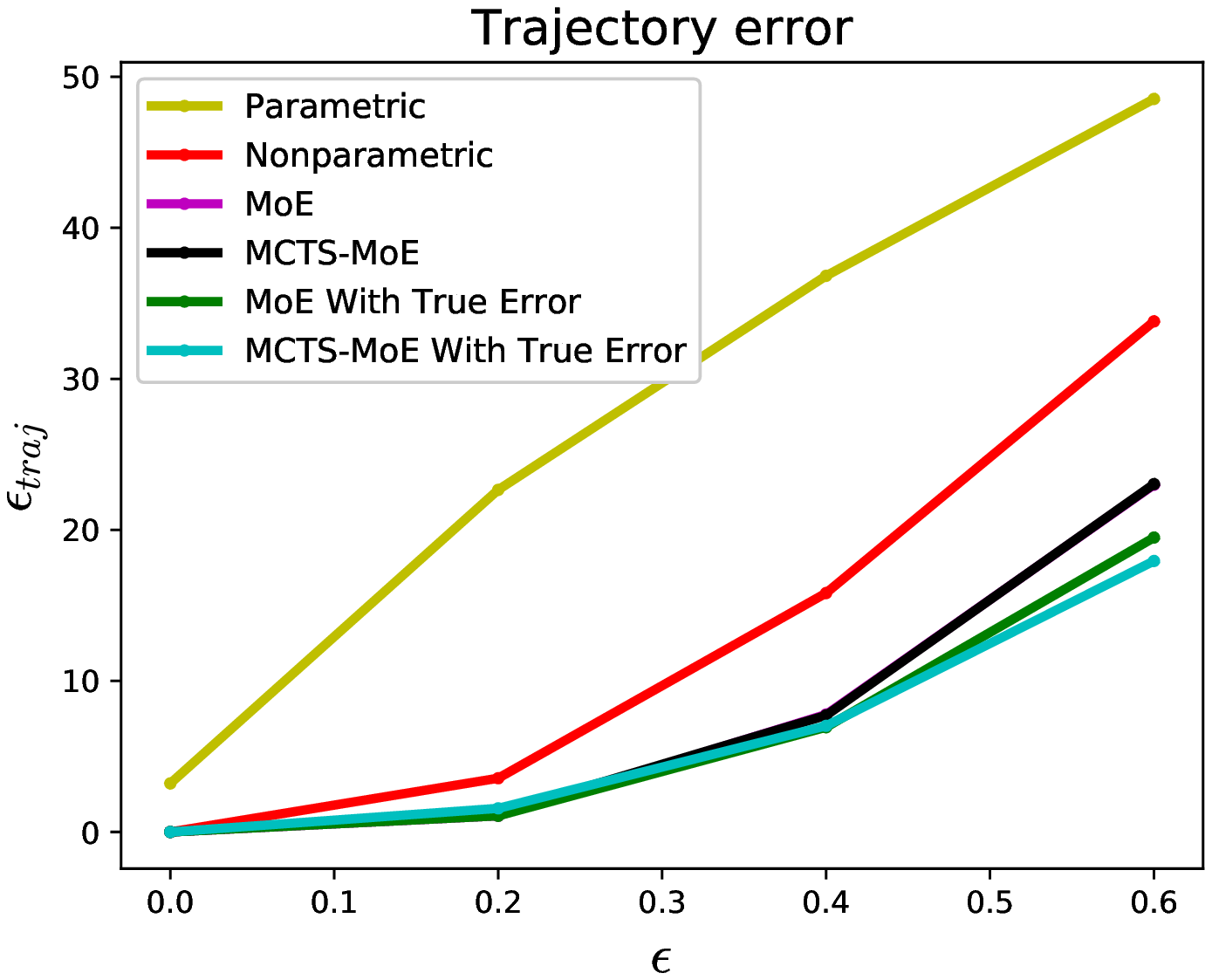}}
\subcaptionbox{Cancer\label{fig:cancer_eval_policy_value_est_err_squared}}
{\includegraphics[width=0.23\textwidth]{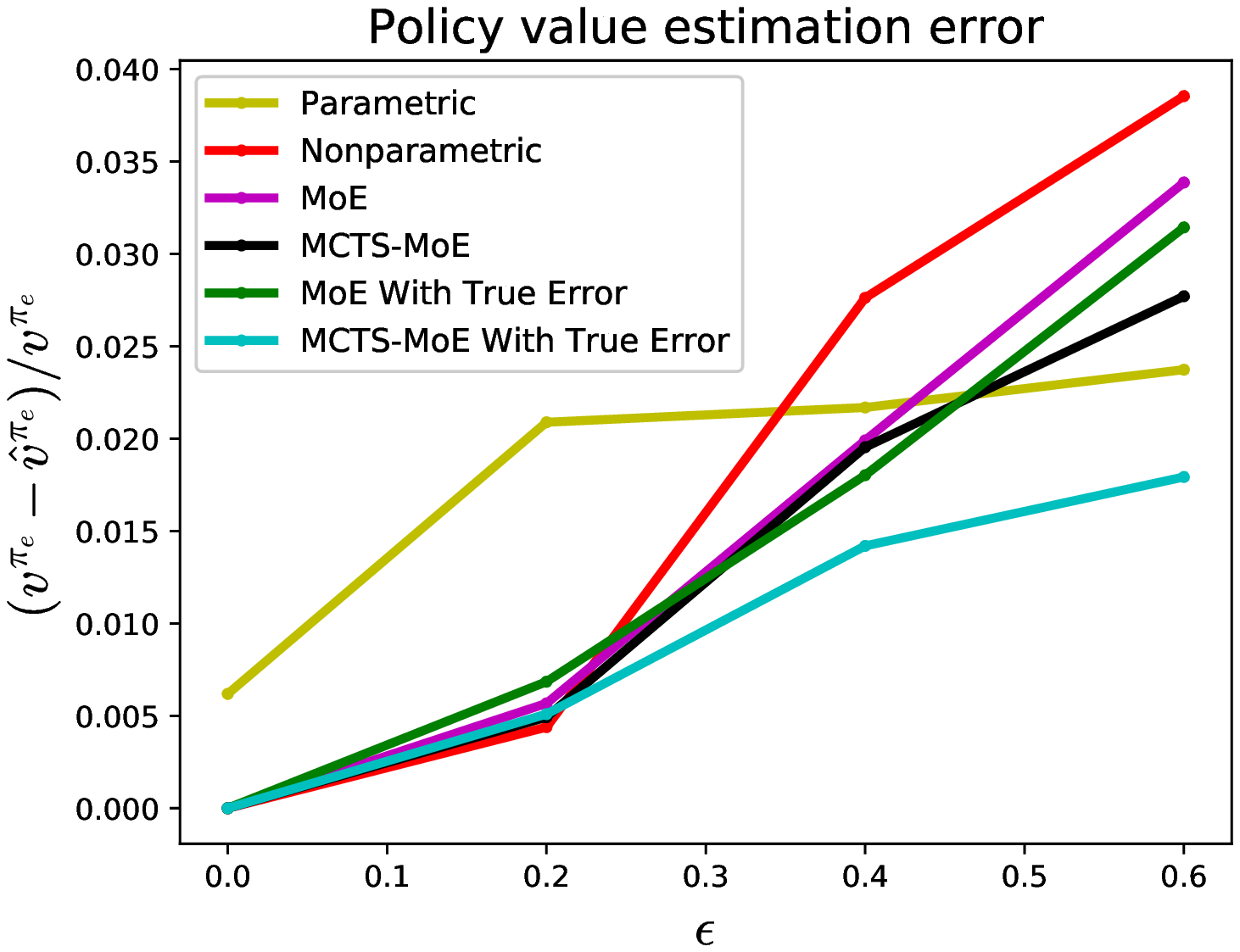}}
\subcaptionbox{HIV\label{fig:hiv_traj_err}}
{\includegraphics[width=0.23\textwidth]{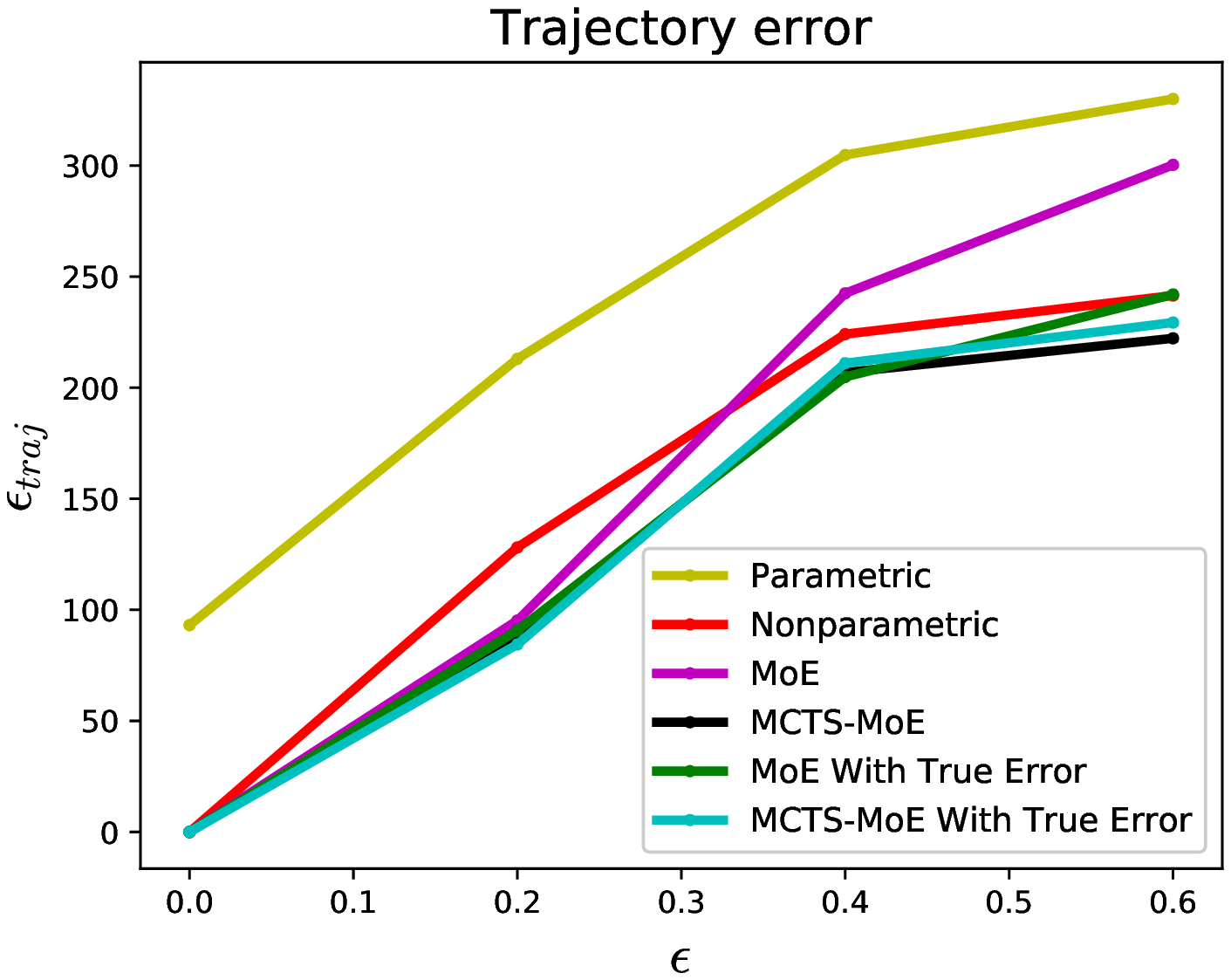}}
\subcaptionbox{HIV\label{fig:hiv_eval_policy_value_est_err_squared}}
{\includegraphics[width=0.23\textwidth]{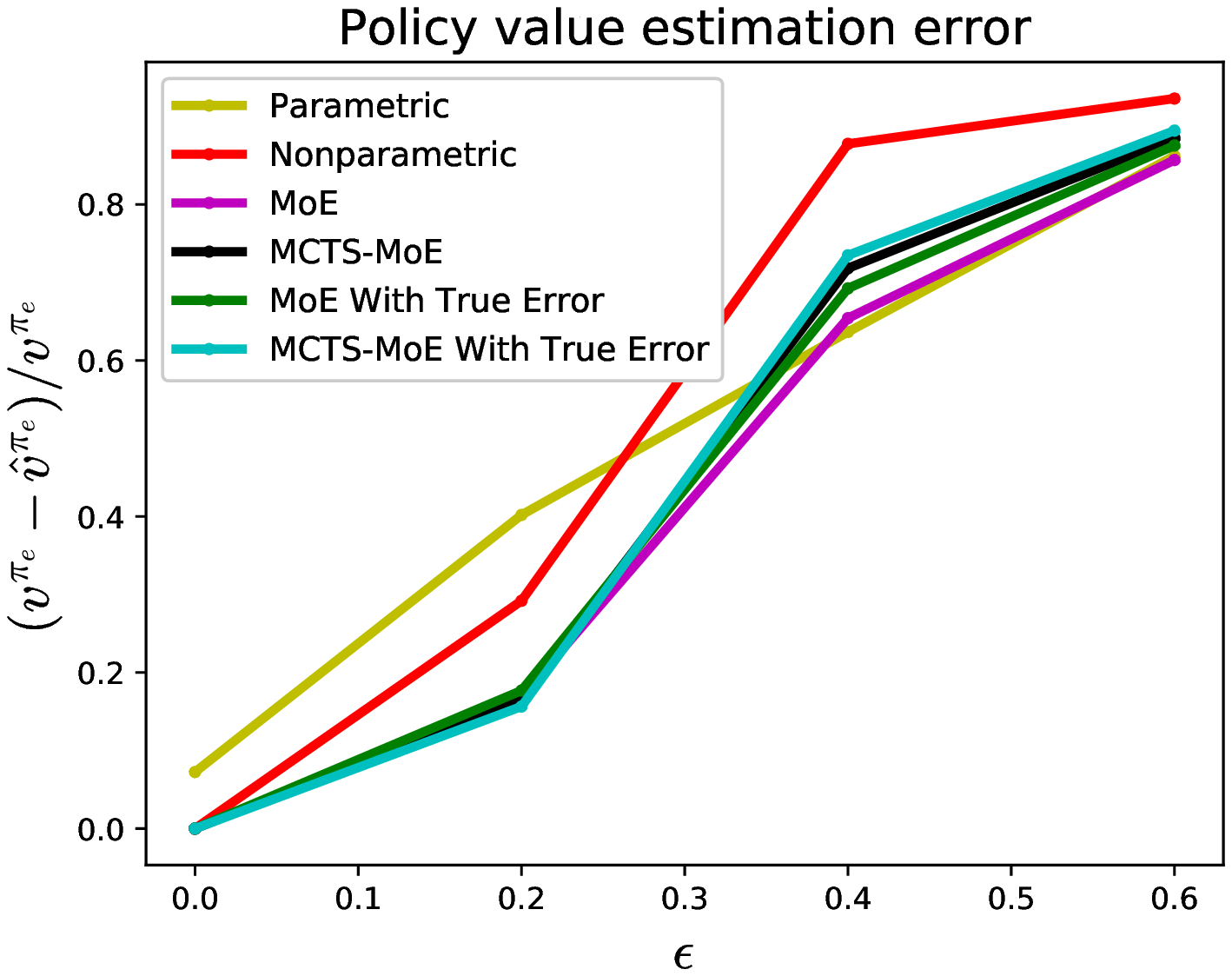}}
\caption{\textbf{Medical simulators.} By combining the advantages of the parametric and nonparametric models, the MoE model can outperform both individual models in terms of simulating accurate trajectories and estimating the evaluation policy value.}
\label{fig:benchmark_results}
\end{figure}

We test the performance of both the greedy MoE simulator and the planning simulator. For each of the two MoE simulators we test their performance both when using the estimates for the errors derived in Sections \ref{sec:estimating_errors_non_parametric} and \ref{sec:estimating_errors_parametric}, and their performance when they are given access to the true errors of each model. Access to the true error is unrealistic for real data, but it allows us to investigate the quality of our error estimates and how much the MoE simulator can potentially be improved by using better error estimates.

\paragraph{Metrics.}
For both domains we compared the MoE simulators with the parametric and nonparametric models using two metrics. The first is the difference between the trajectories simulated using the the models and the trajectory which under the true environment. We define the trajectory error as $\varepsilon_{\textrm{traj}} \coloneqq \sum_{t=0}^T \Delta(x_t,\hat{x}_t)$, where $\hat{x}_t$ is the state prediction at time $t$ using the tested model. The second is the RMSE for the evaluation policy value estimate.

\paragraph{Baselines.}
We compare the evaluation policy value prediction of the MoE simulators with the performance of both the parametric and nonparametric models individually, as well as common importance sampling based OPE methods.

\paragraph{Performance of the different models.}

We first compared the model based estimators with importance sampling based OPE methods. IS based estimators tend to perform poorly with limited data. Indeed, the value estimation errors for all IS based methods were at least an order of magnitude larger than all model based methods. In Table \ref{table:comparision_with_is} we present a comparison of the RMSE of the value estimation for both domains and behavior policy with $\epsilon=0.4$ to demonstrate that for these domains IS methods perform significantly worse than model based methods (comparison to additional IS estimators is presented in Appendix \ref{sec:IS_extended}). This huge difference in performance is consistent across all experimental parameters we tested, and we therefore focus the rest of the results in this section on comparing model based OPE methods only.

\begin{table}
\vspace{-0.5em}
\caption{Relative value estimation RMSE for medical simulators.}
\label{table:comparision_with_is}
\small
\centering
\begin{tabular}{c|cccccc}
\toprule
\multicolumn{1}{c}{} & 
\multicolumn{1}{c}{$M_{\text{p}}$} &
\multicolumn{1}{c}{$M_{\text{np}}$} &
\multicolumn{1}{c}{$M_{\text{MoE}}$} &
\multicolumn{1}{c}{$M_{\text{MCTS-MoE}}$} &
\multicolumn{1}{c}{IS} &
\multicolumn{1}{c}{WDR} \\
\midrule
Cancer & 0.021 & 0.027 & 0.020 & \textbf{0.019} & 1.0 & 0.22 \\
HIV & 0.65 & 0.88 & 0.64 & \textbf{0.63} & 1.0 & 0.99 \\
\bottomrule
\end{tabular}
\end{table}

In Figure \ref{fig:benchmark_results} we present the trajectory simulation error ($\varepsilon_{\textrm{traj}}$) for the different models, and the RMSE of the estimated evaluation policy value using the different models and evaluation methods.  For the cancer domain (Figures \ref{fig:cancer_traj_err} and \ref{fig:cancer_eval_policy_value_est_err_squared}), the MoE model generates more realistic trajectories (smaller trajectory error) which result in better policy value estimates. Introducing planning further improves the value estimation performance, especially when the MoE has access to the true errors.

For the HIV domain (Figures \ref{fig:hiv_traj_err}-\ref{fig:hiv_eval_policy_value_est_err_squared}), the MoE simulator achieves lower trajectory errors, except for high values of $\epsilon$ in the behavior policy. We believe this is due to the inaccuracy of the model error estimation, since the MoE simulator achieves lower trajectory error when given access to the true model errors. We note, however, that the introduction of planning allows the MoE simulator to obtain low trajectory error even without access to the true model errors. For this domain, we see that for large values of $\epsilon$, small trajectory error does not necessarily result in smaller policy evaluation error. This is because the Euclidean distance penalizes error in all state dimensions equally, while in this case one state dimension is much more relevant to the reward.  In Appendix \ref{sec:different_metric_for_hiv}, we demonstrate how choosing a domain-appropriate metric can further improve the value estimation prediction. Furthermore, in Appendix \ref{sec:consistency_experiments} we empirically test our simulators for consistency and show that the value estimation error for both domains decreases as the number of trajectories is increased.

\section{Discussion}

In this paper we demonstrated the effectiveness of a method for combining a parametric and nonparametric model for performing OPE.  Our method is consistent (under mild assumptions) in the limit of infinite data, while effectively choosing sequences of imperfect models to reduce the error in the value estimate with finite data.

Our methods take advantage of techniques used in planning for off-policy evaluation.  While MCTS worked well as a planner for our tasks, one can imagine substituting any future modern planner.  Similarly, we found that our approach for estimating the errors of the models allowed the planner to make choices for more accurate off-policy evaluation than the baselines.  That said, improving the quality of error estimates is an important direction for research.  A very related question is of what metric is appropriate for a particular domain.  We imagine that in many domains expert knowledge may be available; there are also interesting directions in optimizing that metric from data.

Finally, in this work we assumed that the transition and reward functions are deterministic.  We emphasize that our approach can be applied to stochastic domains without modification to the planning algorithm; the only change would be defining an appropriate error estimate---for example, rather than defining the transition error as a distance between the true and simulated next state, we might define it as the distance between the state distribution under the true environment and the one predicted by the models.  Producing accurate estimates in the stochastic setting is an interesting direction for future work.

\bibliography{bib}
\bibliographystyle{icml2019}

\clearpage

\begin{appendices}

\appendix
\section{Model Selection Algorithms}
\label{appendix:model_selection_algorithms}

\begin{algorithm}
   \caption{Greedy model selection}
   \label{alg:greedy_model_selection}
    \begin{algorithmic}
    \FUNCTION{GreedyMoeModelSelection$(s_t, a_t)$}
    \STATE $\hat{\varepsilon}_{\mathrm{t, np}} \leftarrow$ Eq. \ref{eq:non_parametric_err_est}
    \STATE $\hat{\varepsilon}_{\mathrm{t, p}}\leftarrow$ Eq. \ref{eq:parametric_err_est}
    \IF{$\hat{\varepsilon}_{\mathrm{np}} (x_t^{(n)},a_t^{(n)}) < \hat{\varepsilon}_{\mathrm{p}} (x_t^{(n)},a_t^{(n)})$}
        \STATE // Return nonparametric model
        \STATE {\bfseries Return} $(\hat{f}_{t,np},\hat{f}_{r,np})$
    \ELSE
        \STATE // Return parametric model
        \STATE {\bfseries Return} $(\hat{f}_{t,p},\hat{f}_{r,p})$
    \ENDIF
    \ENDFUNCTION
\end{algorithmic}
\end{algorithm}

\begin{algorithm*}
   \caption{MCTS-MoE model selection}
   \label{alg:mcts_model_selection}
\begin{multicols}{2}
\begin{algorithmic}
    \FUNCTION{MctsMoeModelSelection$(s_t, a_t)$}
        \STATE create root node $\nu_t$ with state $(s_t, a_t)$
        \STATE $\delta(\nu_t) \leftarrow 0$ // State error bound for node
        \STATE $\delta_g(\nu_t) \leftarrow 0$ // Return error bound for node
        \STATE $\tau(\nu) \leftarrow 0$ // Time-steps from root node
        \WHILE{within computational budget}
            \STATE $\nu_l \leftarrow$ TreePoicy$(\nu_t)$
            \STATE $V \leftarrow$ DefaultPolicy$(\nu_l)$
            \STATE Backup$(\nu_l, V)$
        \ENDWHILE
    \STATE {\bfseries Return} Model$(\underset{\nu' \in \text{children of } \nu_t}{\arg\max}\widetilde{Q}(\nu'))$
    \ENDFUNCTION
    \\[1\baselineskip]
    \FUNCTION{TreePolicy$(\nu)$}
        \WHILE{$\nu$ is not terminal}
            \IF{$\nu$ not fully expanded}
                \STATE {\bfseries Return} Expand$(\nu)$
            \ELSE
                \STATE $\nu \leftarrow \underset{\nu' \in \text{children of } \nu}{\arg\max} \frac{Q(\nu)}{N(\nu)} + c_e \sqrt{\frac{2 \ln{N(\nu)}}{N(\nu')}}$
            \ENDIF
            \STATE {\bfseries Return} $\nu$
        \ENDWHILE
    \ENDFUNCTION
    \\[1\baselineskip]
    \FUNCTION{Expand$(\nu)$}
        \STATE add a new child $\nu'$ to $\nu$
        \IF{$\nu$ has no children}
            \STATE Model$(\nu') \leftarrow$ GreedyMoeModelSelection$(s(\nu), a(\nu))$
        \ELSE
            \STATE Model$(\nu') \leftarrow$ model not yet tried in $\nu$
        \ENDIF
        \STATE $(s(\nu'), a(\nu')) \leftarrow (\hat{f}_{t,\text{Model}(\nu')}, \pi_e(\hat{f}_{t,\text{Model}(\nu')}))$
        \STATE $\varepsilon_t(\nu'), \varepsilon_r(\nu') \leftarrow$ ComputeErrors(Model$(\nu')$)
        \STATE $N(\nu') \leftarrow 0$ // Times node was visited
        \STATE $Q(\nu') \leftarrow 0$ // Total reward of all rollouts through node
        \STATE $\widetilde{Q}(\nu') \leftarrow 0$ // Rollout with highest reward for node
        \STATE $\tau(\nu') \leftarrow \tau(\nu) + 1$
        \STATE $\delta(\nu') \leftarrow L_t \cdot \delta(\nu) + \varepsilon_t(\nu')$
        \STATE $\delta_g(\nu') \leftarrow \delta_g(\nu) + \gamma^{\tau(\nu')} \left( \varepsilon_r(\nu') + L_t \cdot \delta(\nu') \right)$
        \STATE {\bfseries Return} $\nu'$
    \ENDFUNCTION
    \\[1\baselineskip]
    \FUNCTION{DefaultPolicy$(\nu)$}
        \STATE $(s^*, a^*) \leftarrow (s(\nu), a(\nu))$
        \STATE $\tau^* \leftarrow \tau(\nu)$
        \STATE $\delta^* \leftarrow \delta(\nu)$
        \STATE $\delta_g^* \leftarrow \delta_g(\nu)$
        \WHILE{s in not terminal}
            \STATE Model $\leftarrow$ GreedyMoeModelSelection$(s, a)$
            \STATE $s \leftarrow \hat{f}_{t,Model}(s,a)$
            \STATE $\varepsilon_t^*, \varepsilon_r^* \leftarrow$ ComputeErrors(Model)
            \STATE $a \leftarrow \pi_e(s)$
            \STATE $\tau^* \leftarrow \tau^* + 1$
            \STATE $\delta^* \leftarrow L_t \cdot \delta^* + \varepsilon_t^*$
            \STATE $\delta_g^* \leftarrow \delta_g^* + \gamma^{\tau^*} \left( \varepsilon_r^* + L_t \cdot \delta^* \right)$
        \ENDWHILE
        \STATE {\bfseries Return} $-\delta_g^*$
    \ENDFUNCTION
    \\[1\baselineskip]
    \FUNCTION{Backup$(\nu, V)$}
        \WHILE{$\nu$ is not null}
            \STATE $N(\nu) \leftarrow N(\nu)+1$
            \STATE $Q(\nu) \leftarrow Q(\nu)+V$
            \STATE $\widetilde{Q}(\nu) \leftarrow \max(\widetilde{Q}(\nu), V)$
            \STATE $\nu \leftarrow$ parent of $\nu$
        \ENDWHILE
    \ENDFUNCTION
    \\[1\baselineskip]
    \FUNCTION{ComputeErrors(Model)}
        \IF{Model = parametric}
            \STATE $\varepsilon_t \leftarrow$ Eq. \ref{eq:parametric_err_est}
            \STATE $\varepsilon_r \leftarrow$ Eq. \ref{eq:parametric_reward_err_est}
        \ELSE
            \STATE // Model = nonparametric
            \STATE $\varepsilon_t \leftarrow$ Eq. \ref{eq:non_parametric_err_est}
            \STATE $\varepsilon_r \leftarrow$ Eq. \ref{eq:non_parametric_reward_err_est}
        \ENDIF
        \STATE {\bfseries Return} $\varepsilon_t, \varepsilon_r$
    \ENDFUNCTION
\end{algorithmic}
\end{multicols}
\end{algorithm*}

In this section we provide the two algorithms used to choose the model in the MoE simulator. The functions \emph{GreedyMoeModelSelection} in Algorithm \ref{alg:greedy_model_selection} and \emph{MctsMoeModelSelection} in Algorithm \ref{alg:mcts_model_selection} can be substituted with \emph{ChooseModel} in Algrorithm \ref{alg:moe_simulator} in the main text.

Algorithm \ref{alg:greedy_model_selection} is straight forward and simply returns the model with the smaller immediate estimated transition error. This algorithm could also use a weighted sum of both the transition and reward error, but that choice would require choosing a tuning parameter which controls the relative importance of the transitions and rewards accuracy.

Algorithm \ref{alg:mcts_model_selection} is based on the standard upper confidence bound for trees (UCT) algorithm \cite{coulom2006efficient, browne2012survey}. We note once again that the domain over which the MCTS algorithm plans  is not the same  domain as the RL environment. The “states” for the  MCTS  algorithm are  state-action  pair in the RL domain, and the ”actions” are choosing either the parametric or nonparametric model.

The value of a rollout for the planner is minus the return error bound derived in Theorem \ref{theorem:total_return_error} in the main text, $-\delta_g$. Because of the compounding effect of the state error bound, $\delta(t)$, the value of $\delta$ for each node must be rolled forward for all nodes which results in the main modifications to the standard UCT algorithm in, mainly in functions \emph{Expand} and \emph{DefaultPolicy}.

A tuning parameter of the UCT algorithm is the exploration constant, $c_e$, which controls how frequently the algorithm should explore branches which appear not promising if they have not been explored enough. When the rewards are bounded between 0 and 1, a standard choice for $c_e$ is $1/\sqrt{2}$. Because we don't know a priori how large the errors might be, we continuously update the exploration parameter such that $c_e=\max{\hat{\varepsilon}_t}/\sqrt{2}$.

\section{Proof of Lemma \ref{lemma:state_error}}
\label{sec:proof_of_state_error_lemma}

We first restate Lemma \ref{lemma:state_error}.

\paragraph{Lemma \ref{lemma:state_error}}
Let $\varepsilon_t(t)$ be the transition estimation error bound for the chosen model at time-step $t$,
\begin{equation}
\varepsilon_t(t) \geq \Delta(\hat{x}_{t+1}, f_t(\hat{x}_{t}, a_{t}))
\end{equation}
The state error at time-step $t$ is: 
\begin{equation}
    \delta(t) \coloneqq \Delta(x_t, \hat{x}_t) \leq \sum_{t'=0}^{t-1} (L_t)^{t'} \varepsilon_t(t-t'-1)
\end{equation}
where $L_t$ is the Lipschitz constant of the transition function, $f_t$.

\emph{Proof.} We prove Lemma \ref{lemma:state_error} by induction. The state prediction error at time $t$ is bounded by:

\begin{align}
    \delta(t) & \coloneqq \Delta(x_t, \hat{x}_t) \\ \nonumber
    & \leq \Delta(x_t, f_t(\hat{x}_{t-1}, a_t)) + \Delta(f_t(\hat{x}_{t-1}, a_{t-1}), \hat{x}_t) \\ \nonumber
    & \leq L_t \delta(t-1) + \varepsilon_t(t-1),
\end{align}

Where the first inequality is a consequence of the triangle inequality. By definition, $\delta(1) \leq \varepsilon(0)$. Therefore

\begin{align}
    \delta(t) & \coloneqq \Delta(x_t, \hat{x}_t) \\ \nonumber
    & \leq L_t \delta(t-1) + \varepsilon_t(t-1) \\ \nonumber
    & \leq L_t (L_t \delta(t-2) + \varepsilon_t(t-2)) + \varepsilon(t-1) \\ \nonumber
    & ... \\ \nonumber
    & \leq \sum_{t'=0}^{t-1} (L_t)^{t'} \varepsilon_t(t-t'-1),
\end{align}

completing the proof.

\section{Proof of Consistency}
\label{sec:consistency}

In this section we are going to prove MoE simulator (Algorithm \ref{alg:moe_simulator}) with MCTS model selection is a consistent estimator i.e. the return error goes to zero when the number of samples collected from behavior policy goes to infinity. We assume the planning error of MCTS is bounded by $\epsilon_{\text{planning}}$ where the objective of planning is to maximize:
\begin{align}
\label{eq:mcts_objective}
     - L_r \sum_{t=0}^{T-t_0} \gamma^t \sum_{t'=0}^{t-1} (L_t)^{t'} \hat{\varepsilon}_{t} (t_0+t-t'-1) \nonumber \\
     - \sum_{t=0}^{T-t_0} \gamma^t \hat{\varepsilon}_{r}(t_0+t)
\end{align}
for any input state action pair $(s_{t_0},a_{t_0})$.

\begin{assumption}
\label{assm:behavior_converage}
(Coverage of behavior policy) For a data set $\mathcal{D}$ with n samples collected from behavior policy and any given state $x$ and action $a$, let $\text{rad}_n$ be $\min_{x_t^{(i)} \in \mathcal{D}, a_t^{(i)}=a} \Delta(x,x_t^{(i)})$. Then $\lim_{n\to \infty} \text{rad}_n = 0$.
\end{assumption}

\begin{assumption}
\label{assm:C_converage}
(Coverage of radius C) There exist an $N$ such that for any $n>N$, for any n sample collected from behavior policy and any state $x$ and action $a$, the chosen radius $C$ satisfy that there is at least one sample in data set is within distance $C$ of $x$ and matches the action $a$.
\end{assumption}

\begin{assumption}
\label{assm:parametric_model_Lipschitz}
(Lipschitz continuity of parametric model) Functions $\hat{f}_t$ and $\hat{f}_r$ in parametric model class are $L$-Lipschitz with $L < \infty$. 
\end{assumption}

\begin{lemma}
\label{lemma:nonparametric_consistency}
Under assumptions \ref{assm:behavior_converage} and \ref{assm:parametric_model_Lipschitz}, Let $n$ be the number of samples collected from behavior policy. For any x:
\begin{align*}
   \lim_{n \to \infty} \varepsilon_{t,np}(x) = 0, \quad \lim_{n \to \infty} \varepsilon_{r,np}(x) = 0\\
   \lim_{n \to \infty} \hat{\varepsilon}_{t,np}(x) = 0, \quad  \lim_{n \to \infty} \hat{\varepsilon}_{r,np}(x) = 0 
\end{align*}
\end{lemma}

\begin{proof}
Let $x_t^{(i)}$ be the state closest to $x$ whose action $a_t^{(i)}$ equals a.
\begin{align}
    \varepsilon_{t,np}(x) = \Delta(f_t(x,a), f_t(x_t^{(i)},a)) \\
    \le L_t \Delta(x,x_t^{(i)}) \le L_t \text{rad}_n \\
    \varepsilon_{r,np}(x) = \Delta(f_r(x,a), f_r(x_t^{(i)},a)) \\
    \le L_r \Delta(x,x_t^{(i)}) \le L_r \text{rad}_n
\end{align}
Thus $0 \le \lim_{n \to \infty} \varepsilon_{t,np}(x) \le L_t \lim_{n \to \infty} \text{rad}_n = 0$. So $\lim_{n \to \infty} \varepsilon_{t,np}(x) = 0$, similarly $\lim_{n \to \infty} \varepsilon_{r,np}(x) = 0$. For the estimated error:
\begin{align}
    \hat{\varepsilon}_{t,np}(x) =  \hat{L}_t \Delta(x,x_t^{(i)}) \le \hat{L}_t \text{rad}_n \\
    = \max_{i \neq j} \frac{\Delta(x_{t'+1}^{(i)}, x_{t''+1}^{(j)})}{\Delta(x_{t'}^{(i)}, x_{t''}^{(j)})} \text{rad}_n\\
    \le L_t \text{rad}_n
\end{align}
Similarly, we have $\lim_{n \to \infty} \hat{\varepsilon}_{t,np}(x) = 0$ and $\lim_{n \to \infty} \hat{\varepsilon}_{r,np}(x) = 0$
\end{proof}

A direct conclusion following from this claim and Theorem \ref{theorem:total_return_error} is that the non-parametric model is a consistent estimator. 

\begin{lemma}
\label{lemma:parametric_bounded}
Let $L_{\hat{f}_t}$ be the Lipschitz constant of the parametric model $\hat{f}_t$, and $L_{\hat{f}_r}$ be the Lipschitz constant of $\hat{f}_r$. 
\begin{align}
    \varepsilon_{t,p}(x) \le \hat{\varepsilon}_{t,p}(x) + L_t\text{rad}_n + L_{\hat{f}_t} \text{rad}_n \\
    \varepsilon_{r,p}(x) \le \hat{\varepsilon}_{r,p}(x) + L_r\text{rad}_n + L_{\hat{f}_r} \text{rad}_n
\end{align}
\end{lemma}
\begin{proof}
Let $x_t^{(i)}$ be the state closest to $x$ whose action $a_t^{(i)}$ equals a.
\begin{eqnarray}
    && \varepsilon_{t,p}(x) = \Delta(f_t(x,a), \hat{f}_t(x,a)) \\
    &\le& \Delta(f_t(x,a), f_t(x_t^{(i)},a)) \nonumber \\ 
    && + \Delta(f_t(x_t^{(i)},a), \hat{f}_t(x_t^{(i)},a)) \nonumber \\
    && + \Delta(\hat{f}_t(x_t^{(i)},a), \hat{f}_t(x,a)) \\
    &\le& L_t\text{rad}_n + \Delta(f_t(x_t^{(i)},a), \hat{f}_t(x_t^{(i)},a)) \nonumber \\ && + L_{\hat{f}_t} \text{rad}_n 
\end{eqnarray}
Since the closest sample $x_t^{(i)}$ is within distance $C$ of the state of interest $x$ by Assumption \ref{assm:C_converage},
\begin{align}
    \Delta\left(f_t(x_t^{(i)},a), \hat{f}_t(x_t^{(i)},a)\right) = \Delta\left(\hat{f}_t(x_t^{(i)},a), x_{t+1}^{(i)}\right) \\
    \le \max \Delta \left(\hat{f}_t(x_{t'}^{(i)},a), x_{t'+1}^{(i)}\right) = \hat{\varepsilon}_{t,p}
\end{align}
So we finished the proof for $\varepsilon_{t,p}(x)$. Similarly we can show $\varepsilon_{r,p}(x) \le \hat{\varepsilon}_{r,p}(x) + L_r\text{rad}_n + L_{\hat{f}_r} \text{rad}_n$
\end{proof}

Now we are going to prove Theorem \ref{theorem:consistency}:

\setcounter{theorem}{1}
\begin{theorem} 
(Restated) Under the assumptions \ref{assm:behavior_converage}, \ref{assm:C_converage}, \ref{assm:parametric_model_Lipschitz} in our appendix, assuming planning error $\epsilon_{\text{planning}} = o(1)$, the MoE simulator with MCTS model selection is a consistent estimator of policy value of $\pi_e$.
\end{theorem}

\begin{proof}
By assuming the planning error of MCTS is bounded by $\epsilon_{\text{planning}}$, we have that the return of chosen node will be no less than the return of nonparametric model minus $\epsilon_{\text{planning}}$.
\begin{eqnarray}
    && \max_{\nu' \in \text{children of } \nu} \widetilde{Q}(\nu') \\
    &\ge& - L_r \sum_{t=0}^{T-t_0} \gamma^t \sum_{t'=0}^{t-1} (L_t)^{t'} \hat{\varepsilon}_{t, np} (t_0+t-t'-1) \nonumber \\
    && - \sum_{t=0}^{T-t_0} \gamma^t \hat{\varepsilon}_{r, np}(t_0+t) - \epsilon_{\text{planning}} \\
    &\ge& - K \text{rad}_n - \epsilon_{\text{planning}}
\end{eqnarray}
where $K$ is some constant independent of sample size n. By MCTS algorithm, we have that the return of chosen node is:
\begin{eqnarray}
    && \max_{\nu' \in \text{children of } \nu} \widetilde{Q}(\nu') \\
    &=& - L_r \sum_{t=0}^{T-t_0} \gamma^t \sum_{t'=0}^{t-1} (L_t)^{t'} \hat{\varepsilon}_{t, \text{MCTS}} (t_0+t-t'-1) \nonumber \\
    && - \sum_{t=0}^{T-t_0} \gamma^t \hat{\varepsilon}_{r,\text{MCTS}} (t_0+t) 
\end{eqnarray}
where $\varepsilon_{t, \text{MCTS}} (t)$ and $\varepsilon_{r, \text{MCTS}} (t)$ is the transition and reward error of the model selected by MCTS MoE model selection algorithm at each planning step. Thus 
\begin{align}
    L_r \gamma \hat{\varepsilon}_{t, \text{MCTS}} (t_0) + \hat{\varepsilon}_{r, \text{MCTS}} (t_0) \le -\widetilde{Q}(\nu') \nonumber \\ \le  K \text{rad}_n + \epsilon_{\text{planning}}
\end{align}

Then we can bound the estimated one step transition and reward error of the chosen model by
\begin{align}
    \hat{\varepsilon}_{t, \text{MCTS}}(t_0) \le K' (\text{rad}_n + \epsilon_{\text{planning}}) \\
    \hat{\varepsilon}_{r, \text{MCTS}}(t_0) \le K' (\text{rad}_n + \epsilon_{\text{planning}})
\end{align}
where $K'$ is some other constant independent with sample size $n$. Now we need to bound the true one step transition and reward error of the chosen model $\varepsilon_{t, \text{MCTS}}(t_0)$ and $\varepsilon_{t, \text{MCTS}}(t_0)$. By Lemma \ref{lemma:nonparametric_consistency} we know that 
we can bound it for non-parametric model for any state:
\begin{align}
\varepsilon_{t,np}(x) \le L_t \text{rad}_n, \quad \varepsilon_{r,np}(x) \le L_r \text{rad}_n
\end{align}
and Lemma \ref{lemma:parametric_bounded} show that 
\begin{align}
    \varepsilon_{t,p}(x) \le \hat{\varepsilon}_{t,p}(x) + L_t\text{rad}_n + L_{\hat{f}_t} \text{rad}_n \\
    \varepsilon_{r,p}(x) \le \hat{\varepsilon}_{r,p}(x) + L_r\text{rad}_n + L_{\hat{f}_r} \text{rad}_n
\end{align}

Then for both model we have that
\begin{align}
    \varepsilon_{t}(x) \le \hat{\varepsilon}_{t}(x) + K'' \text{rad}_n  \\
    \varepsilon_{r}(x) \le \hat{\varepsilon}_{r}(x) + K'' \text{rad}_n
\end{align}
for some constant $K''$. Therefore for the chosen model, we can bound its one step transition error and reward error.
\begin{align}
    \varepsilon_{t,\text{MCTS}}(x) &\le \hat{\varepsilon}_{t,\text{MCTS}}(x) + K'' \text{rad}_n \nonumber \\
    &= O(\text{rad}_n) + O(\epsilon_{\text{planning}}) \\
    \varepsilon_{r,\text{MCTS}}(x) &\le \hat{\varepsilon}_{r,\text{MCTS}}(x) + K'' \text{rad}_n \nonumber \\
    &= O(\text{rad}_n) + O(\epsilon_{\text{planning}})
\end{align}

Combining this with Theorem \ref{theorem:total_return_error}, we have that the total error of return could be bounded by $O(\text{rad}_n) + O(\epsilon_{\text{planning}})$. Thus, if 
$
    O(\epsilon_{\text{planning}}) = o(1)
$, the total return error will also be bounded by $o(1)$ and MoE simulator with MCTSmodel selection is a consistent estimator.
\end{proof}

\section{Consistency of MCTS-MoE Under Weaker Conditions}
\label{sec:consistency_alternate_model_selection}

In our proof of theorem \ref{theorem:consistency}, we assume that the planning error $\epsilon_{\text{planning}}$ will converge to zero. If that is not true, we can still prove the consistency result with a slightly different variant of Algorithm \ref{alg:greedy_model_selection}. Consider if the condition in line 4 of Algorithm \ref{alg:greedy_model_selection} changes to:
\begin{equation}
    \varepsilon_{t,p}(x) + \alpha_r \varepsilon_{r,p}(x) 
    \le \hat{\varepsilon}_{t,p}(x) + \alpha_r \hat{\varepsilon}_{r,p}(x),
\end{equation}
where the coefficient $\alpha_r$ is a constant factor just determined by the scale of reward and transition function. Then we can show a new theorem about Algorithm \ref{alg:moe_simulator} with both greedy and MCTS model selection are consistent estimators i.e. the return error goes to zero when the number of samples collected from behavior policy goes to infinity. We keep the same assumptions (Assumption \ref{assm:behavior_converage}, \ref{assm:C_converage}, \ref{assm:parametric_model_Lipschitz}) for other parts of algorithm as last section.
\begin{lemma}
\label{lemma:greedy_consistency}
MoE simulator with greedy model selection is a consistent estimator of the policy value of $\pi_e$.
\end{lemma}

\begin{theorem}
\label{theorem:no_assumption_consistency}
MoE simulator with MCTS model selection is a consistent estimator of the policy value of $\pi_e$.
\end{theorem}
Proof sketch: Notice that only when $\hat{\varepsilon}_{t,p}(x) + \alpha_r \hat{\varepsilon}_{r,p}(x) \le \hat{\varepsilon}_{t,np}(x) + \alpha_r \hat{\varepsilon}_{r,np}(x) $ we will select parametric model. Then Lemma \ref{lemma:greedy_consistency} can be proved by showing the greedy model is consistent since the nonparametric model is consistent. Thus we can further prove Theorem \ref{theorem:no_assumption_consistency} by show that the MCTS policy will always choose a model better than greedy selection since greedy selection is the default roll out policy and the environment is deterministic.

We now show the proofs formally. Proof of Lemma \ref{lemma:greedy_consistency}:
\begin{proof}
We are going to show that the error of the return goes to zero as the number of samples goes to infinity. According to Theorem \ref{theorem:total_return_error}, we only need to show that $\varepsilon_{t,\text{greedy}}(t)$ and $\varepsilon_{r,\text{greedy}}(t)$ goes to zero for any time $t$ where $\text{greedy} \in \{p, np\}$ is the model selected by greedy MoE model selection algorithm at time step $t$.

We showed in Lemma \ref{lemma:nonparametric_consistency} that the non-parametric model error $\varepsilon_{t, np}(x)$ and $\varepsilon_{r, np}(x)$ goes to zero when n goes to infinity. Now we are going to show that we will select a parametric model at a given state $x$ only if $\varepsilon_{t,p}(x) + \varepsilon_{r,p}(x)$ will also go to zero.

According to the greedy model selection algorithm, we will only select the parametric model when 
$$\hat{\varepsilon}_{t,p}(x) + \alpha_r \hat{\varepsilon}_{r,p}(x) \le \hat{\varepsilon}_{t,np}(x) + \alpha_r \hat{\varepsilon}_{r,np}(x) $$,
where the coefficient $\alpha_r$ is a constant factor determined by the scale of reward and transition function. According to Lemma \ref{lemma:parametric_bounded},
\begin{eqnarray}
    && \varepsilon_{t,p}(x) + \alpha_r \varepsilon_{r,p}(x) \\
    &\le& \hat{\varepsilon}_{t,p}(x) + \alpha_r \hat{\varepsilon}_{r,p}(x) + O(\text{rad}_n) \\
    &\le& \hat{\varepsilon}_{t,np}(x) + \alpha_r \hat{\varepsilon}_{r,np}(x) + O(\text{rad}_n) \\
    = O(\text{rad}_n)
\end{eqnarray}
Since $\lim_{n \to \infty} \text{rad}_n = 0$, for any chosen model at time step t, $\varepsilon_{t,\text{greedy}}(t)$ and $\varepsilon_{r,\text{greedy}}(t)$ is also $o(1)$. The proof follows from then applying Theorem \ref{theorem:total_return_error}
\end{proof}

Proof of Theorem \ref{theorem:no_assumption_consistency}
\begin{proof}
According to the MCTS MoE model selection algorithm, for any input $(s_{t_0},a_{t_0})$ we will at least have one roll-out trajectory following by the greedy MoE model selection. So the return of the chosen node is at least larger than this:
\begin{eqnarray}
    && \max_{\nu' \in \text{children of } \nu} \widetilde{Q}(\nu') \\
    &\ge& - L_r \sum_{t=0}^{T-t_0} \gamma^t \sum_{t'=0}^{t-1} (L_t)^{t'} \hat{\varepsilon}_{t, \text{greedy}} (t_0+t-t'-1) \nonumber \\
    && - \sum_{t=0}^{T-t_0} \gamma^t\hat{\varepsilon}_{r,\text{greedy}}(t_0+t) \\
    &\ge& - K \text{rad}_n
\end{eqnarray}
where $K$ is some constant independent of sample size n. This follows from the fact that the estimated error of greedy selected model can be bounded by the estimated error of non-parametric model, and further bounded by $O(\text{rad}_n)$. By the MCTS algorithm, we have that the return of chosen node can be expressed as:
\begin{eqnarray}
    && \max_{\nu' \in \text{children of } \nu} \widetilde{Q}(\nu') \\
     &=& - L_r \sum_{t=0}^{T-t_0} \gamma^t \sum_{t'=0}^{t-1} (L_t)^{t'} \hat{\varepsilon}_{t, \text{MCTS}} (t_0+t-t'-1) \nonumber \\
    && - \sum_{t=0}^{T-t_0} \gamma^t\hat{\varepsilon}_{r,\text{MCTS}}(t_0+t) 
\end{eqnarray}
where $\varepsilon_{t, \text{MCTS}} (t)$ and $\varepsilon_{r, \text{MCTS}} (t)$ are the transition and reward error of the model selected by MCTS MoE model selection algorithm. Thus 
\begin{align}
    L_r \gamma \hat{\varepsilon}_{t, \text{MCTS}} (t_0) + \hat{\varepsilon}_{r, \text{MCTS}} (t_0) \le -\widetilde{Q}(\nu') \le  K \cdot \text{rad}_n
\end{align}

Thus, there exist another constant $K'$ such that the one step transition and reward error of the chosen model satisfy that
\begin{align}
    \hat{\varepsilon}_{t, \text{MCTS}}(t_0) \le K' \text{rad}_n \\
    \hat{\varepsilon}_{r, \text{MCTS}}(t_0) \le K' \text{rad}_n
\end{align}
Now we need to bound the true one step transition and reward error of the chosen model $\varepsilon_{t, \text{MCTS}}(t_0)$ and $\varepsilon_{t, \text{MCTS}}(t_0)$. By Lemma \ref{lemma:nonparametric_consistency} we know that 
we can bound it for non-parametric model for any state:
\begin{align}
\varepsilon_{t,np}(x) \le L_t \text{rad}_n, \quad \varepsilon_{r,np}(x) \le L_r \text{rad}_n
\end{align}
and Lemma \ref{lemma:parametric_bounded} show that 
\begin{align}
    \varepsilon_{t,p}(x) \le \hat{\varepsilon}_{t,p}(x) + L_t\text{rad}_n + L_{\hat{f}_t} \text{rad}_n \\
    \varepsilon_{r,p}(x) \le \hat{\varepsilon}_{r,p}(x) + L_r\text{rad}_n + L_{\hat{f}_r} \text{rad}_n
\end{align}

Then for both model we have that
\begin{align}
    \varepsilon_{t}(x) \le \hat{\varepsilon}_{t}(x) + K'' \text{rad}_n  \\
    \varepsilon_{r}(x) \le \hat{\varepsilon}_{r}(x) + K'' \text{rad}_n
\end{align}
for some constant $K''$. Therefore for the chosen model, we can bound its one step transition error and reward error.
\begin{align}
    \varepsilon_{t,\text{MCTS}}(x) \le \hat{\varepsilon}_{t,\text{MCTS}}(x) + K'' \text{rad}_n = O(\text{rad}_n) \\
    \varepsilon_{r,\text{MCTS}}(x) \le \hat{\varepsilon}_{r,\text{MCTS}}(x) + K'' \text{rad}_n = O(\text{rad}_n)
\end{align}
Combining this with Theorem \ref{theorem:total_return_error}, we have that the total error of return could be bounded by $O(\text{rad}_n)$ and goes to zero as n goes to infinity.
\end{proof}

\section{Evaluation of Model Error Estimators}
\label{sec:error_estimators_empirical_evaluation}

In this section we empirically investigate the quality of the estimators we use for the error of the transition function, by analyzing their performance on the example presented in section \ref{sec:2d_gridworld}. In Figure \ref{fig:model_errors:ee_data_N_true} we plot the true error of the nonparametric model as a function of coordinate for the action "North", and compare it with the estimate from Equation \ref{eq:non_parametric_err_est} in the main text, shown in Figure \ref{fig:model_errors:ee_data_N}. Figures \ref{fig:model_errors:ee_param_N_true} and \ref{fig:model_errors:ee_param_N} are the equivalent figures for the parametric model. Comparing the errors shown in Figures \ref{fig:model_errors:ee_data_N_true} and \ref{fig:model_errors:ee_param_N_true} indicates whether the parametric or nonparametric model should be selected, and the correct selection based on the true errors is presented in Figure \ref{fig:model_errors:data_selected_true_N}. Similarly by comparing the errors presented in Figures \ref{fig:model_errors:ee_data_N} and \ref{fig:model_errors:ee_param_N}, we present in Figure \ref{fig:model_errors:data_selected_N} which model our MoE model would actually select. Finally, in Figure \ref{fig:model_errors:correct_selection_N} we compare Figures \ref{fig:model_errors:data_selected_true_N} and \ref{fig:model_errors:data_selected_N} to show if the MoE model would make the correct choice in which model to use. Similar analyses is presented on the right half of Figure \ref{fig:model_errors} for the "East" action.

We see that the nonparametric model has small error for the areas where trajectories in the data pass through, and the error increases with distance from clusters of observations. The simple parametric model, on the other hand, has errors which are uncorrelated with the density of observations (In all other domains we will present in this paper this will not be the case, as we will learn the parametric model from the data, and therefore expect the parametric model to be more accurate in regions where we have observations of transitions). Our estimates for the error follow this general trend, and more importantly they properly identify the model with the smaller error over most of the space (Figures \ref{fig:model_errors:correct_selection_N} and \ref{fig:model_errors:correct_selection_E}).

\begin{figure}[H]
\centering
\subcaptionbox{\label{fig:model_errors:ee_data_N_true}}
{\includegraphics[width=0.115\textwidth]{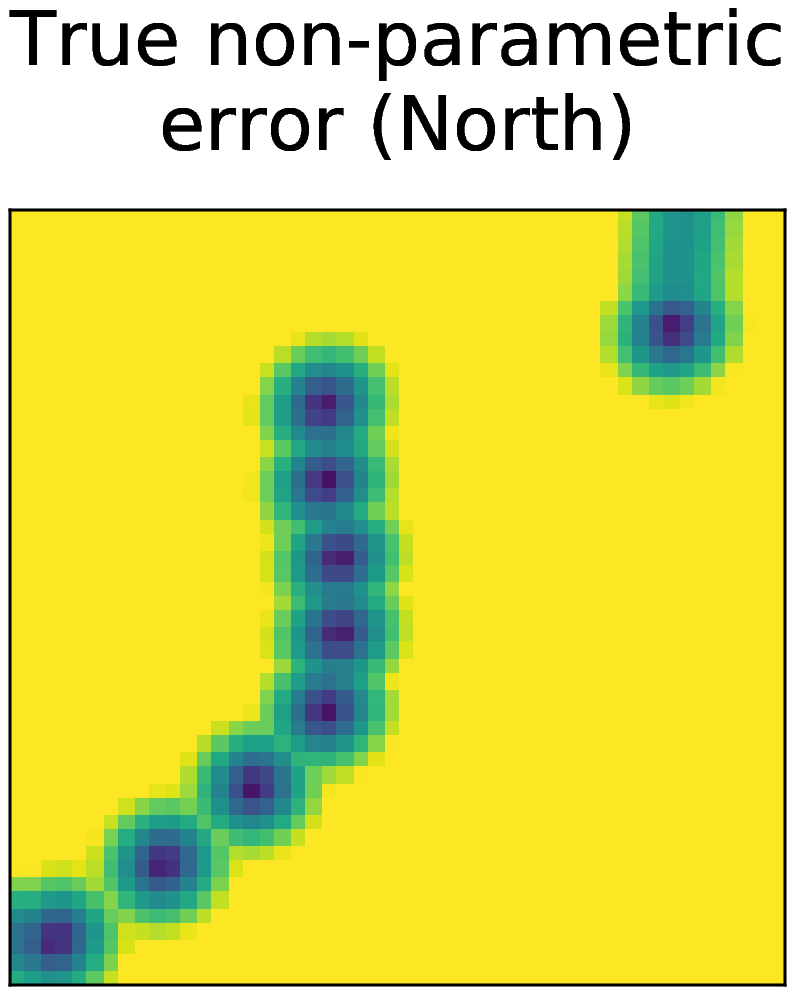}}
\subcaptionbox{\label{fig:model_errors:ee_data_N}}
{\includegraphics[width=0.115\textwidth]{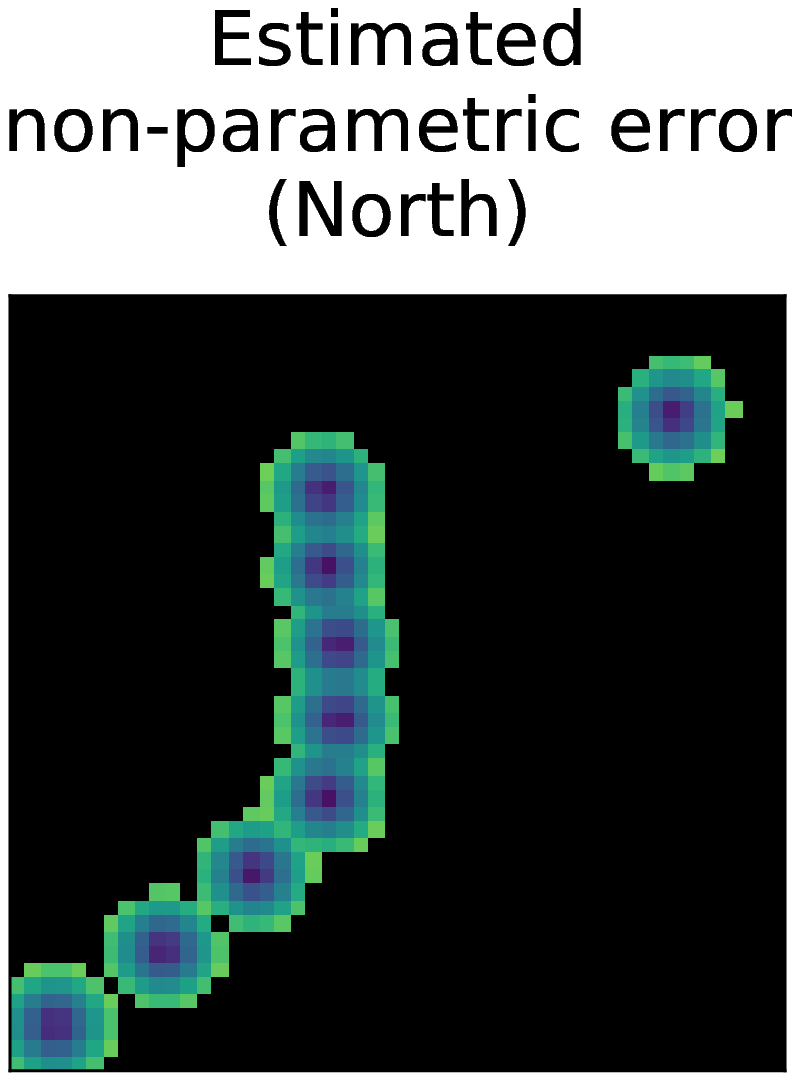}}
\subcaptionbox{\label{fig:model_errors:ee_data_E_true}}
{\includegraphics[width=0.115\textwidth]{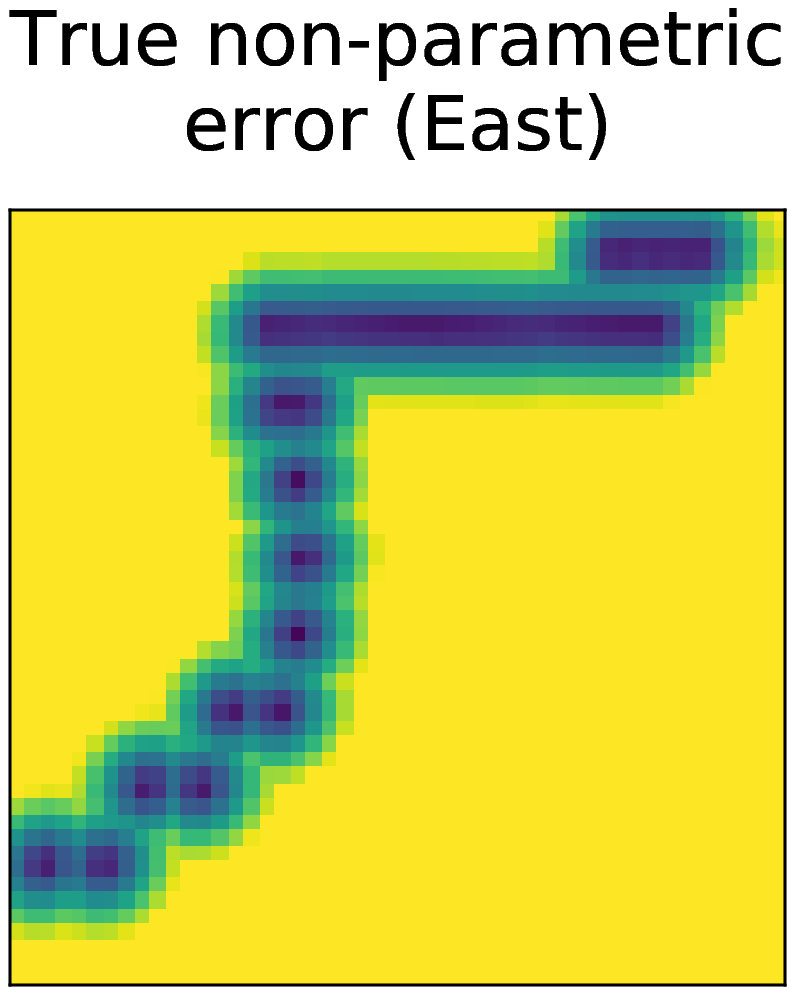}}
\subcaptionbox{\label{fig:model_errors:ee_data_E}}
{\includegraphics[width=0.115\textwidth]{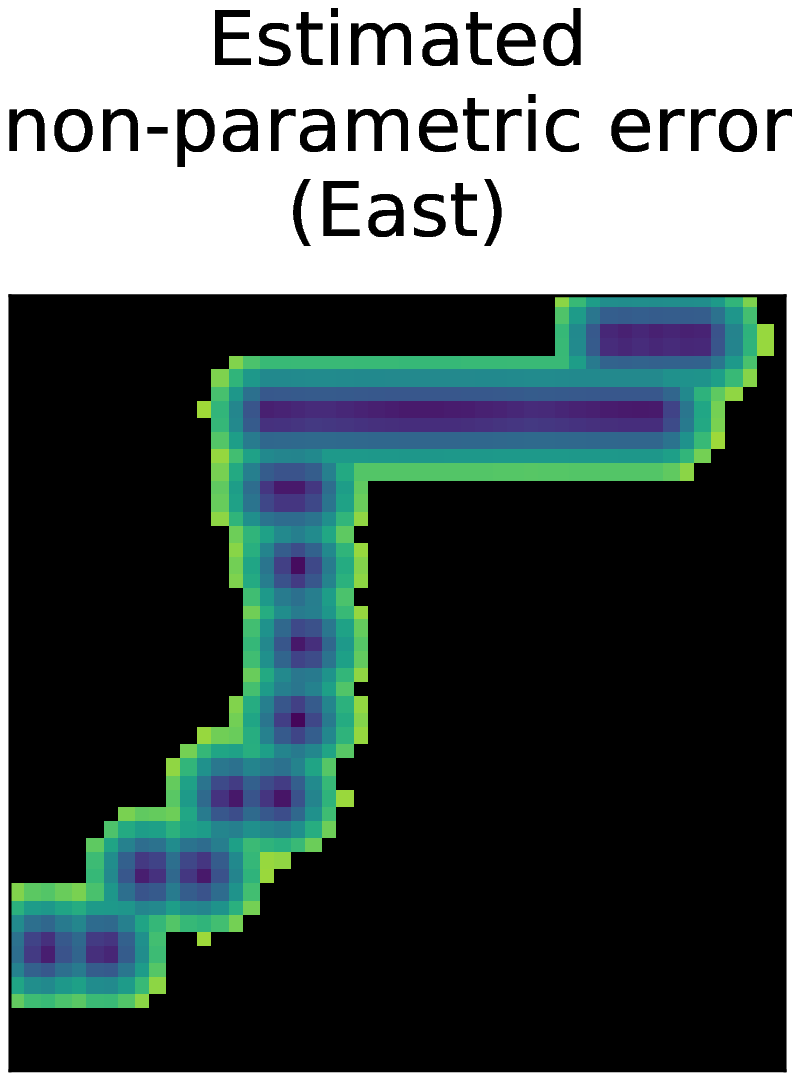}}
\subcaptionbox{\label{fig:model_errors:ee_param_N_true}}
{\includegraphics[width=0.115\textwidth]{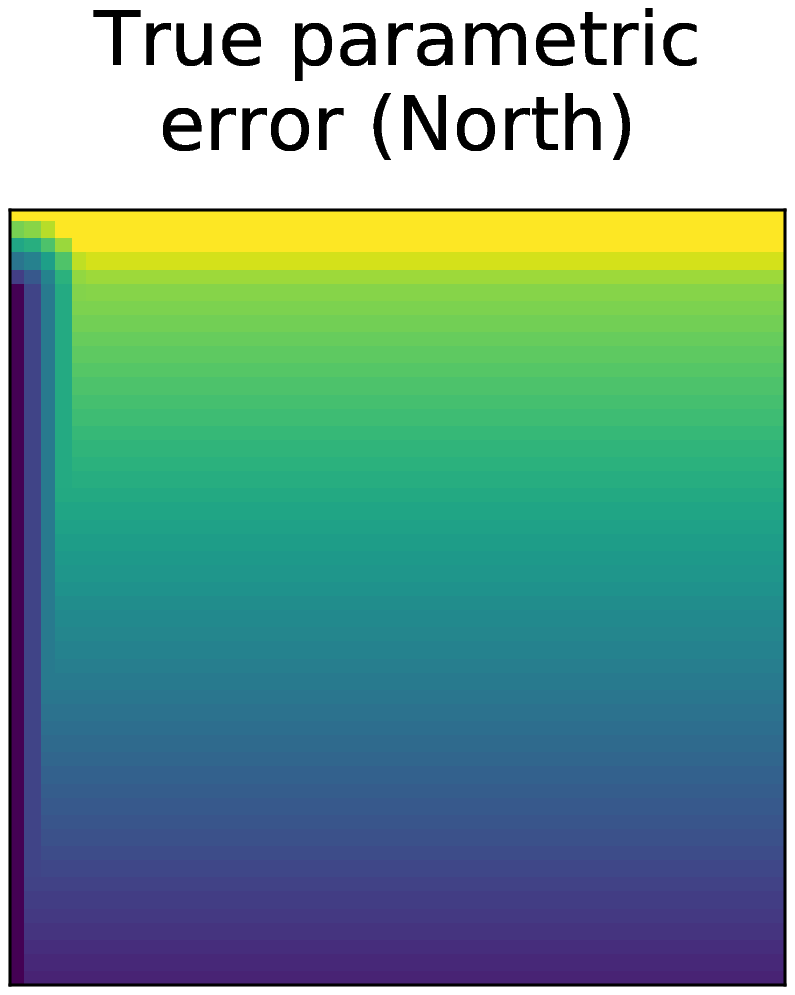}}
\subcaptionbox{\label{fig:model_errors:ee_param_N}}
{\includegraphics[width=0.115\textwidth]{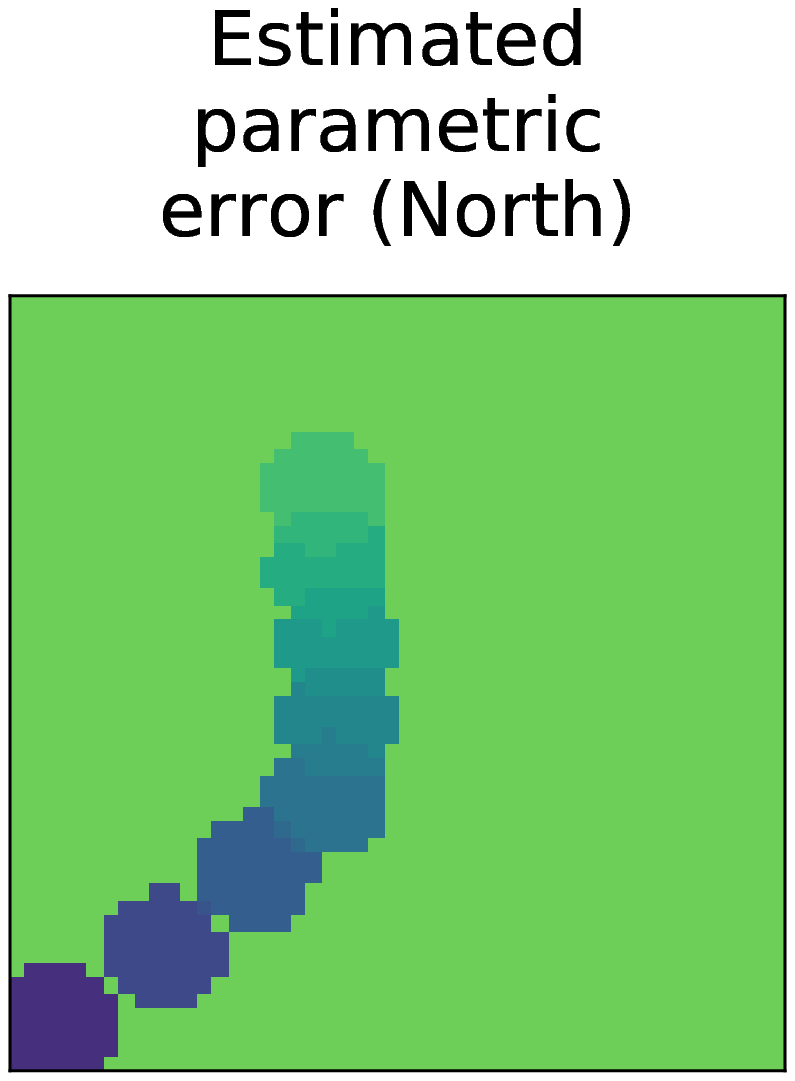}}
\subcaptionbox{\label{fig:model_errors:ee_param_E_true}}
{\includegraphics[width=0.115\textwidth]{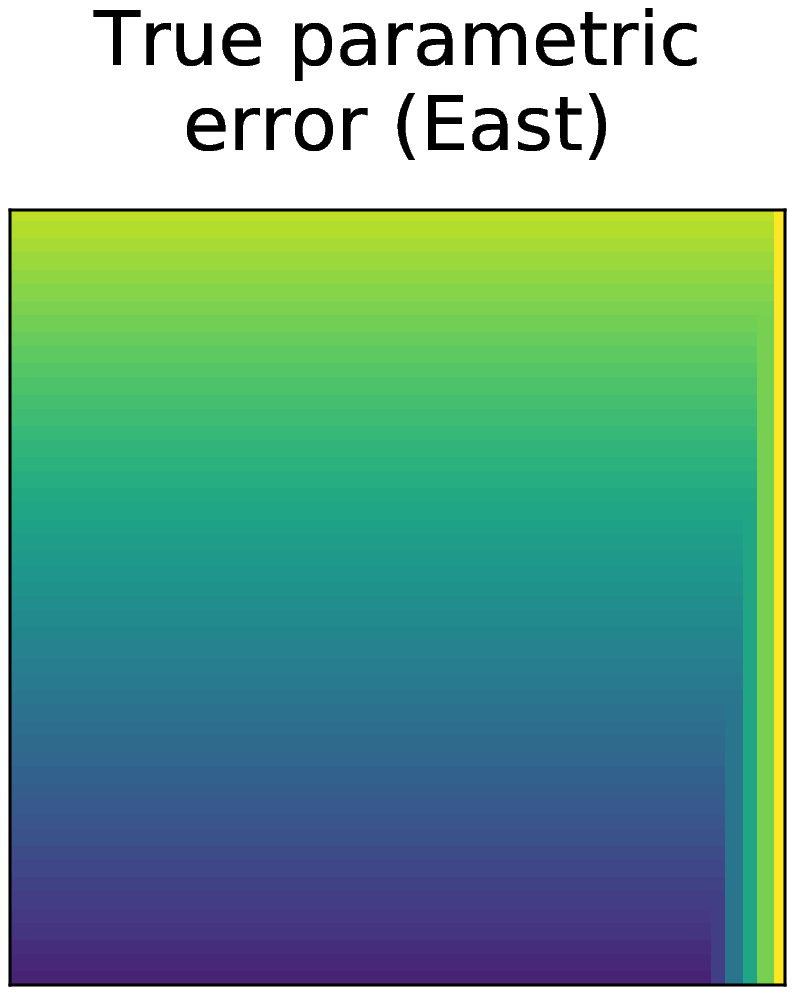}}
\subcaptionbox{\label{fig:model_errors:ee_param_E}}
{\includegraphics[width=0.115\textwidth]{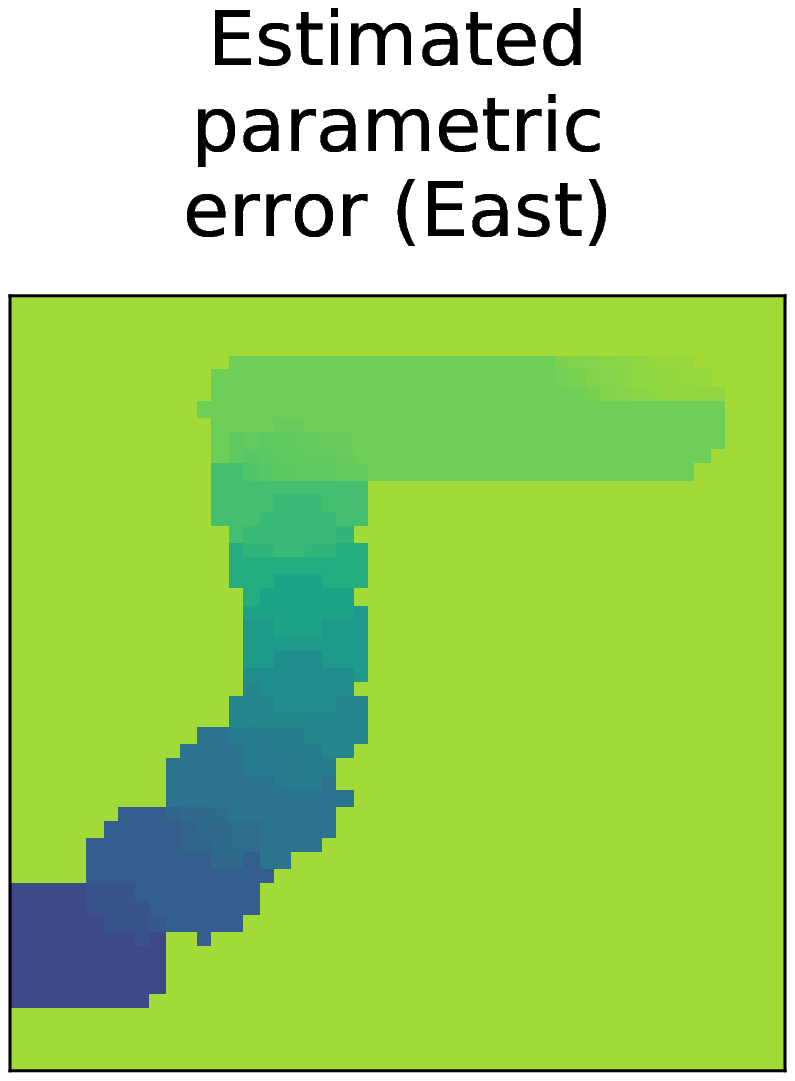}}
\subcaptionbox{\label{fig:model_errors:data_selected_true_N}}
{\includegraphics[width=0.115\textwidth]{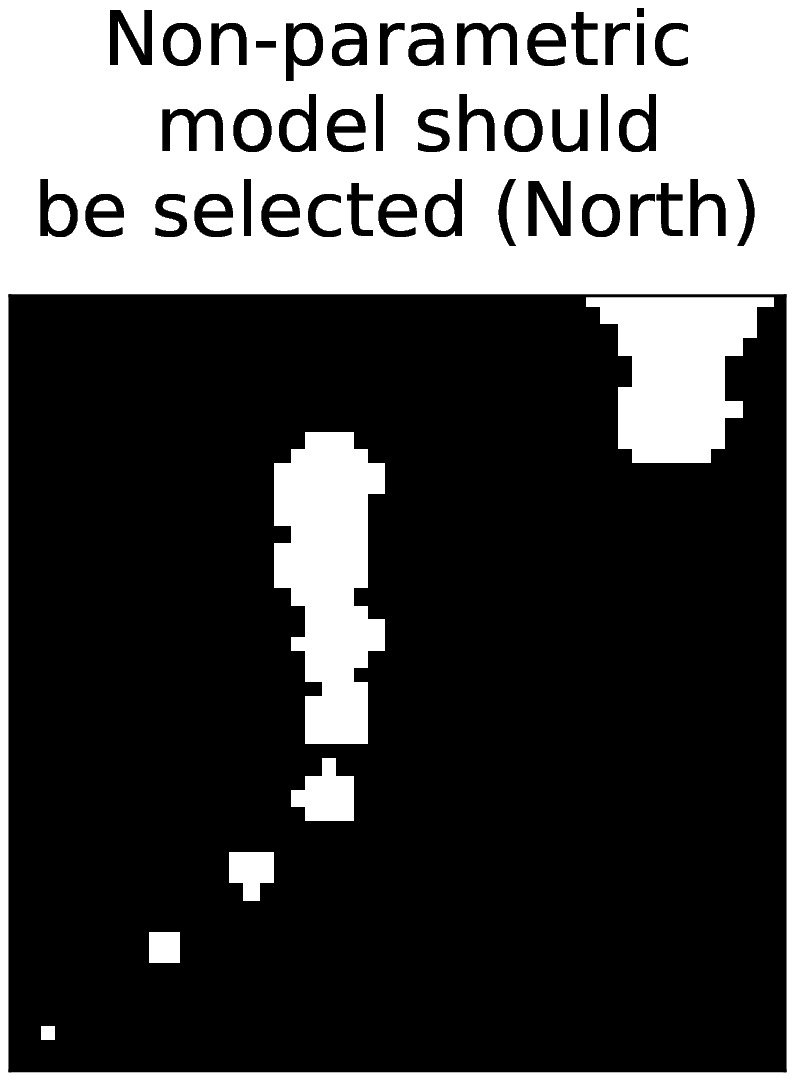}}
\subcaptionbox{\label{fig:model_errors:data_selected_N}}
{\includegraphics[width=0.115\textwidth]{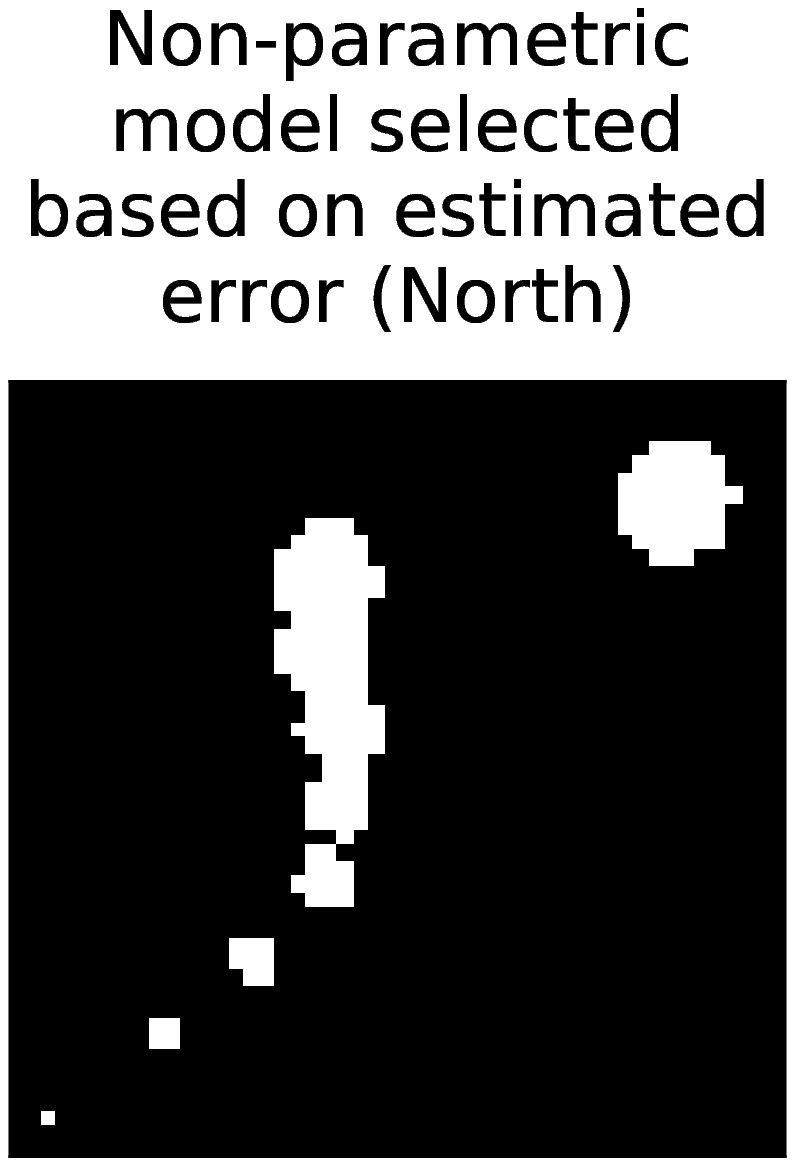}}
\subcaptionbox{\label{fig:model_errors:data_selected_true_E}}
{\includegraphics[width=0.115\textwidth]{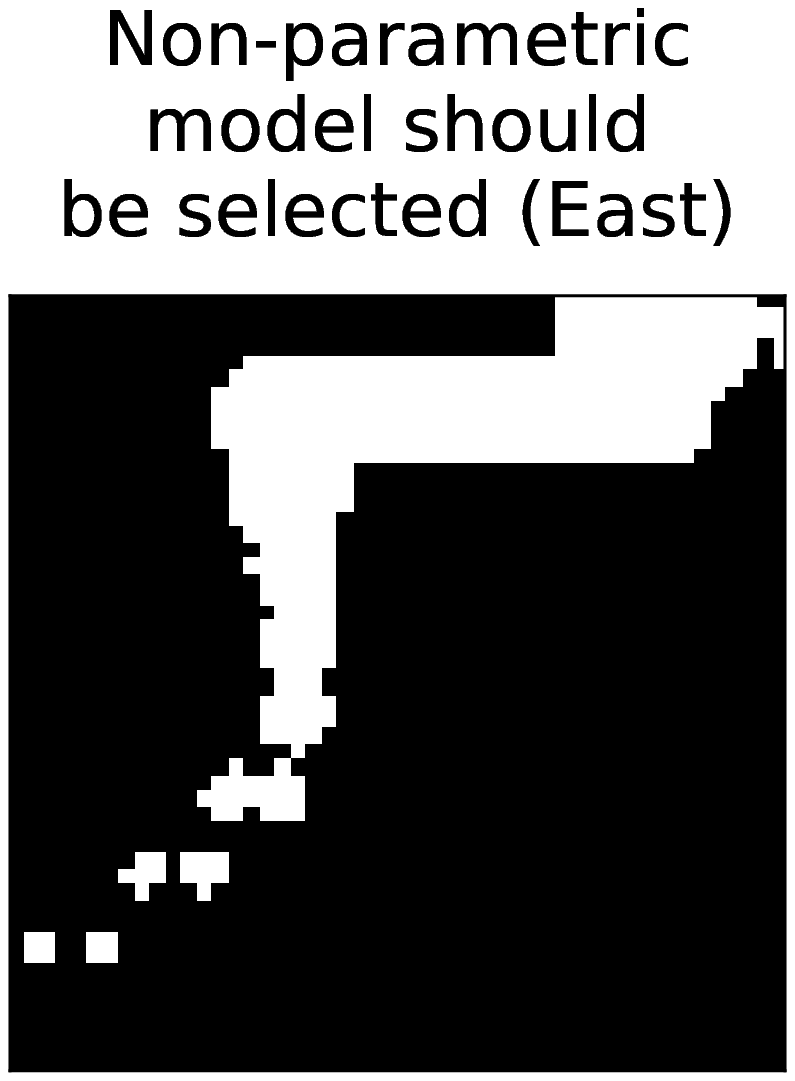}}
\subcaptionbox{\label{fig:model_errors:data_selected_E}}
{\includegraphics[width=0.115\textwidth]{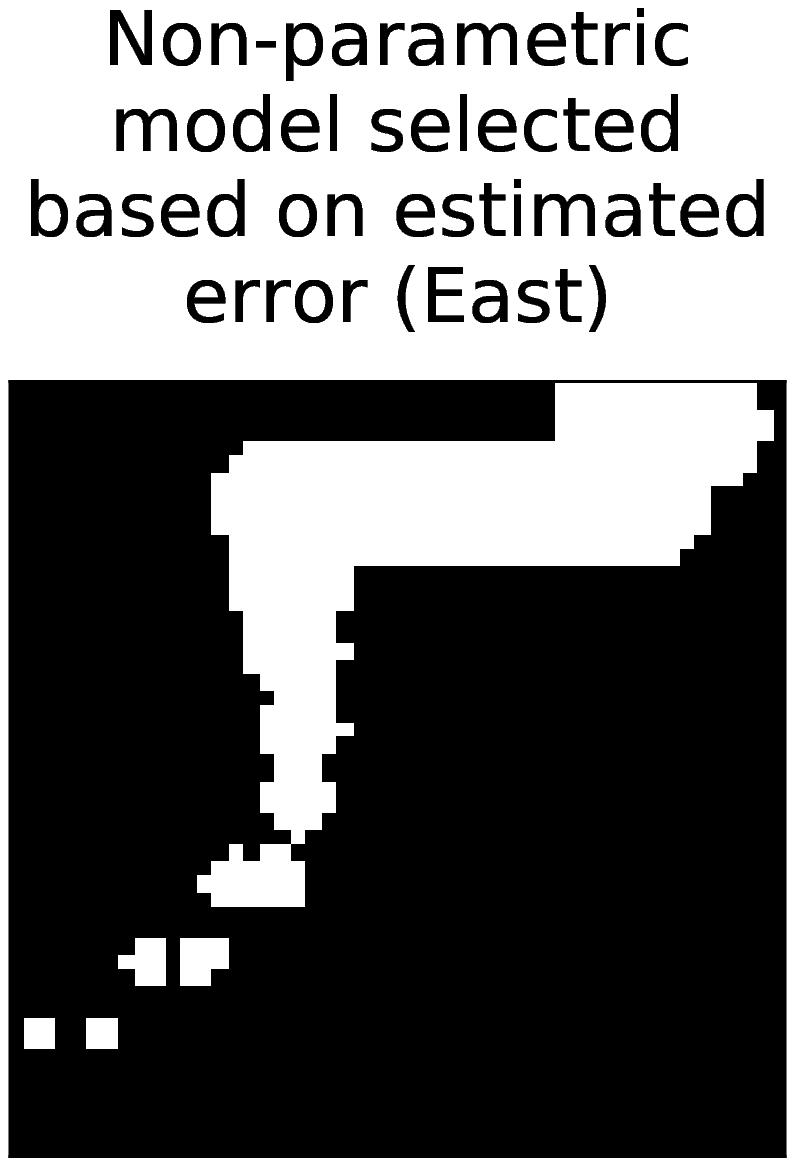}}
\subcaptionbox{\label{fig:model_errors:correct_selection_N}}
{\includegraphics[width=0.115\textwidth]{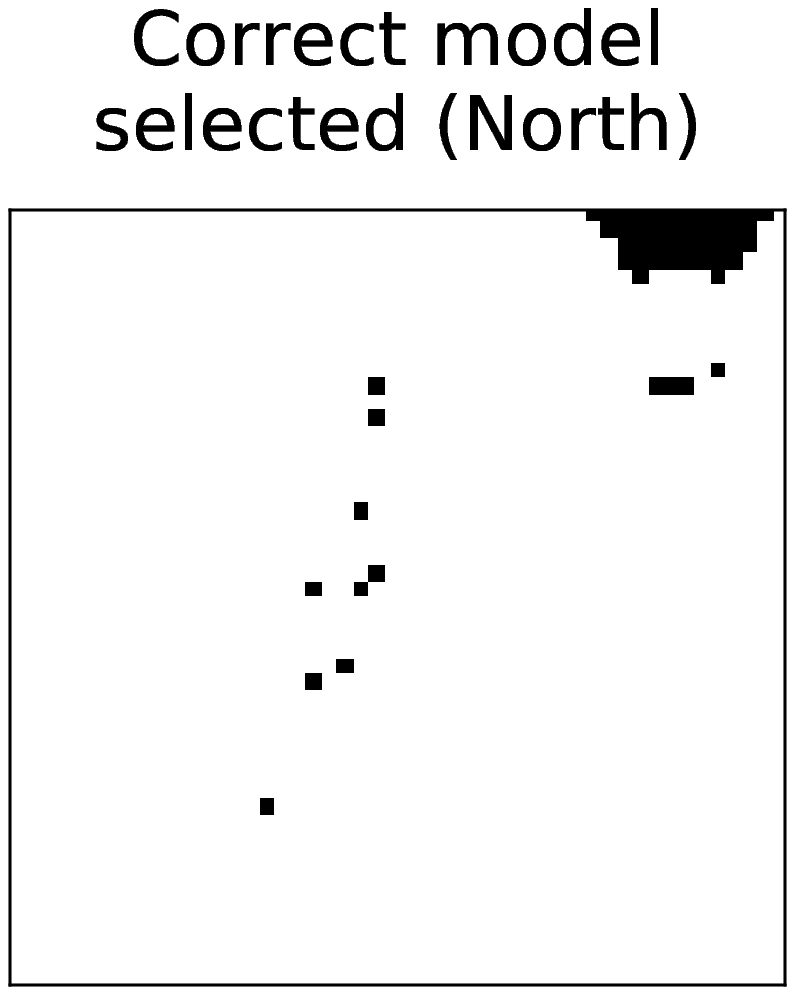}}
\subcaptionbox{\label{fig:model_errors:correct_selection_E}}
{\includegraphics[width=0.115\textwidth]{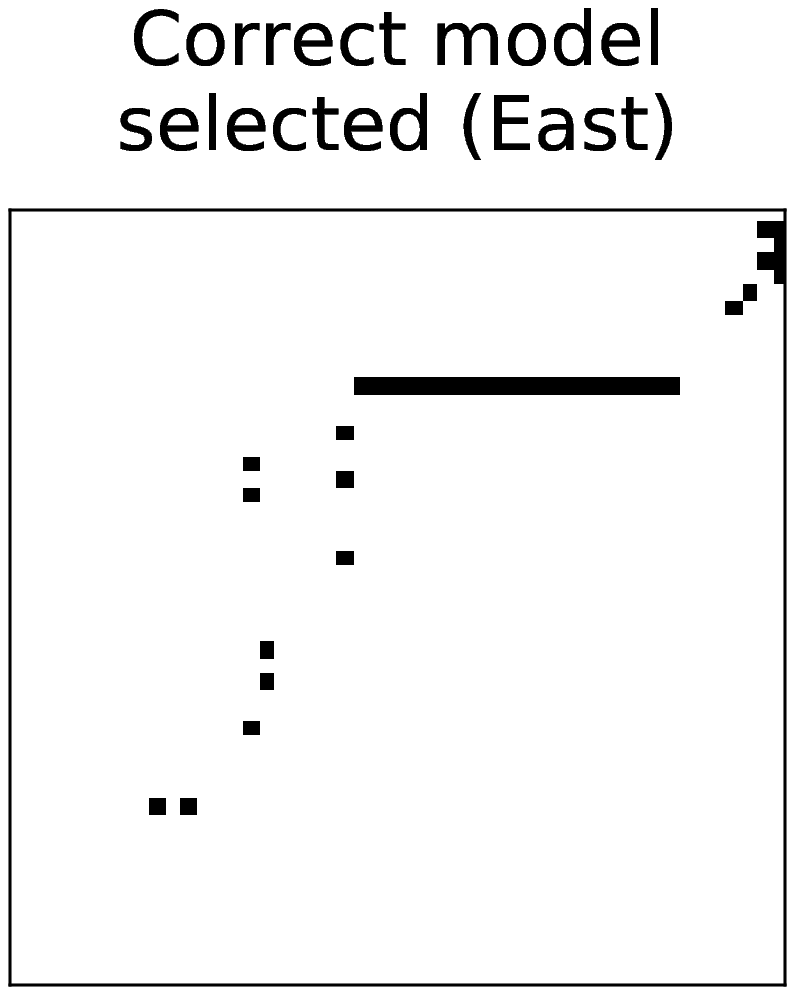}}
\caption{\textbf{Empirical evaluation of error estimates for model errors.}  For the 2D gridworld example described in Section \ref{sec:2d_gridworld} the estimators we use for the model errors resemble the true errors of both the parametric and non-parametric models (a-h). More importantly, these estimators allow the MoE model to correctly select the model with the lower prediction error on the transition (i-n). (All heatmaps figures are presented in the same color scale)}\label{fig:model_errors}
\end{figure}

\section{Experimental Details}
\label{sec:experimental_details}

The dynamics of the cancer domain follow the ODEs presented in \citep{ribba2012tumor} which model the response of cancer cells to treatment. The state space consists of 4 features representing cell counts and medication concentrations, and each time step represent a month in which a clinician may choose between administering a particular treatment or avoiding treatment. The reward at each time step is the total change in diameter of cancerous cells.
To learn the parametric model we fit a linear regression model to predict the dynamics of the states given each action.

The HIV domain is described in \citet{ernst2006clinical}, and consists of 6 parameters describing the state of the patient and 4 possible actions. As the reward function we use the reward described in \citet{ernst2006clinical}. As the parametric model we use a feed-forward neural network with two layers, each consisting of 50 hidden units and a $\tanh$ activation function.

\paragraph{Evaluation and behavior policies.}

For the cancer domain, we test an evaluation policy which treats the patient every month for 10 months, and then stops treatment. As behavior policy we use an $\epsilon$-greedy version of the evaluation policy. For each value of $\epsilon$ we run 500 experiments in which we generate 10 trajectories for learning the models.

In the HIV domain, we use fitted Q iterations to learn an optimal policy. Under this policy --- whose trajectory is shown in Figure \ref{fig:hiv_optimal_policy_true_traj} as the time evolution of the 6 state dimensions --- patients start in a state with a high viral load, which decreases over roughly 70 treatment steps. After the patient is brought to a steady state with low viral load, the continued treatment keeps the patient stabilized. As a behavior policy, we use a policy which is identical to the evaluation policy when the patient is far away from the stable state, and switch to an $\epsilon$-greedy policy around the steady state. This can be thought of as a likely real world scenario where clinicians know how to treat severely ill patients, but are less certain about how to keep them stable in the long run when their condition is not critical. More explicitly, the behavior policy follows the evaluation policy for $\log E<4$, where $E$ is the number of immune effectors, whose evolution is shown in the bottom right plot in Figure \ref{fig:hiv_optimal_policy_true_traj} , and switches to $\epsilon$-greedy when $\log E>4$. For each value of $\epsilon$ we run 100 experiments in which 5 trajectories are generated and used for learning the models.

\begin{figure}[t]
\centering
\subcaptionbox{HIV\label{fig:hiv_optimal_policy_true_traj}}
{\includegraphics[width=0.23\textwidth]{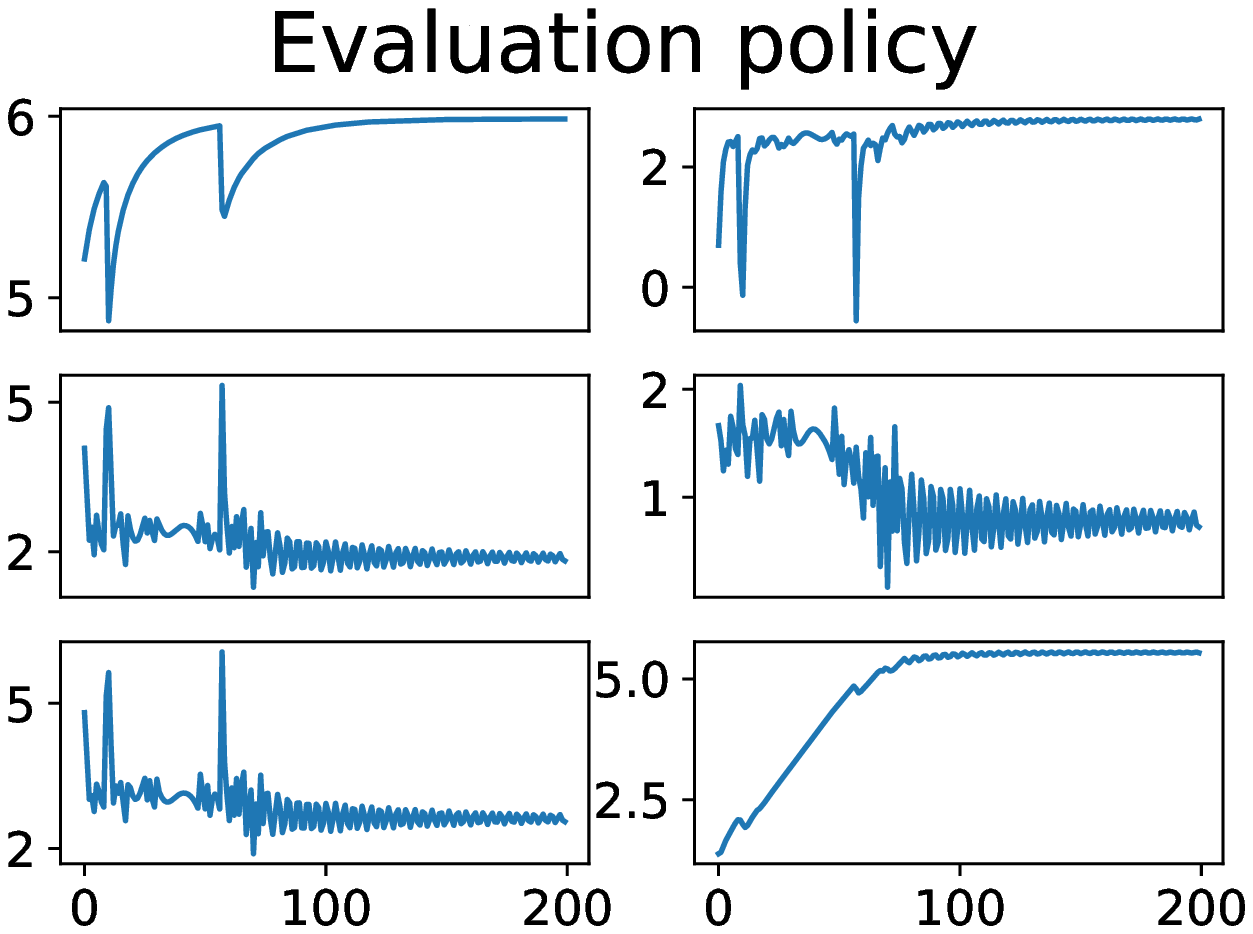}}
\subcaptionbox{HIV\label{fig:hiv_random_behavior_in_ss_behavior_traj}}
{\includegraphics[width=0.23\textwidth]{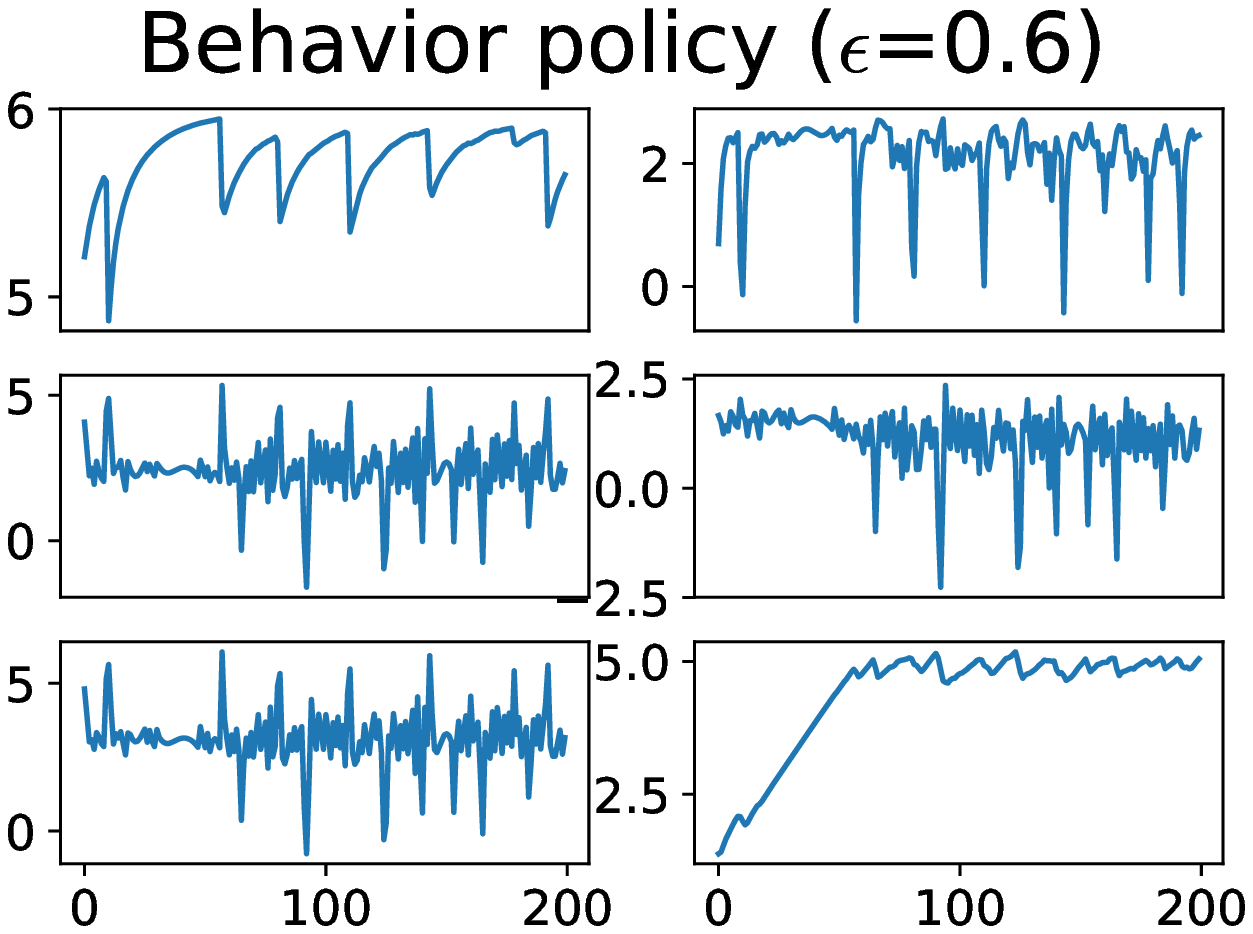}}
\caption{\textbf{Trajectories in the true environment generated by the evaluation and behavior policies.} The behavior policy is similar to the evaluation policy for the initial part of the trajectory (roughly for the first 70 steps) and becomes $\varepsilon$-greedy near the steady state, as can be seen by the more erratic nature of the trajectories for late time steps.}
\label{fig:hiv_eval_and_behavior_trajectories}
\end{figure}

\subsection{Comparison with IS methods}
\label{sec:IS_extended}

\begin{table*}
\vspace{-0.5em}
\caption{$\sqrt{\mathbb{E}[(v^{\pi_e}-\hat{v}^{\pi_e})^2]}/v^{\pi_e}$ ; $(\epsilon=0.4)$}
\label{table:comparison_with_is_extended}
\small
\centering
\begin{tabular}{c|cccccccccc}
\toprule
\multicolumn{1}{c}{} & 
\multicolumn{1}{c}{$M_{\text{p}}$} &
\multicolumn{1}{c}{$M_{\text{np}}$} &
\multicolumn{1}{c}{$M_{\text{MoE}}$} &
\multicolumn{1}{c}{$M_{\text{MCTS-MoE}}$} &
\multicolumn{1}{c}{IS} &
\multicolumn{1}{c}{WIS} &
\multicolumn{1}{c}{PDIS} &
\multicolumn{1}{c}{CWPDIS} &
\multicolumn{1}{c}{DR} &
\multicolumn{1}{c}{WDR} \\
\midrule
Cancer & 0.021 & 0.027 & 0.020 & \textbf{0.019} & 1.0 & 1.0 & 0.55 & 0.22 & 0.87 & 0.22 \\
HIV & 0.65 & 0.88 & 0.64 & \textbf{0.63} & 1.0 & 1.0 & 0.66 & 0.99 & 89.2 & 0.99 \\
\bottomrule
\end{tabular}
\end{table*}

In section \ref{sec:results_medical_simulators} we compared the parametric and nonparametric models, as well as our greedy MoE model to two common importance sampling estimators. In Table \ref{table:comparison_with_is_extended} here we provide additional results for more importance sampling based estimators - standard importance sampling (IS), weighted importance sampling (WIS), per-decision importance sampling (PDIS), consistent weighted per-decision importance sampling (CWPDIS), doubly robust (DR) and weighted doubly robust (WDR) \citep{precup2000eligibility, jiang2016doubly, thomas2016data, thomas2015safe}. The DR and WDR estimators require independent estimates of state values, which we obtain using the parametric model. These results demonstrate that for regimes with limited amount of data, even for moderate trajectory lengths (30 steps for the cancer simulator), all IS based estimators fail due to extremely small effective sample sizes \citep{liu2018breaking, gottesman2018evaluating}, and therefore we must resort to model based estimators.

\subsection{Empirical test for consistency}
\label{sec:consistency_experiments}

In this section we empirically test the consistency of the MoE simulators and demonstrate in Figure \ref{fig:consistency_experiments} that as the number of observed trajectories increases, the value estimation error for both domains decreases across all models. In the cancer domain we see that with access to the true error, the MCTS-MoE consistently outperforms all other methods. For the HIV domain we observe that minimizing the trajectory simulation accuracy does not imply minimizing the value estimation error due to improper choice of metric, as discussed in Appendix \ref{sec:different_metric_for_hiv}.

\begin{figure}[t]
\centering
\subcaptionbox{Cancer\label{fig:cancer_consistency_experiments}}
{\includegraphics[width=0.23\textwidth]{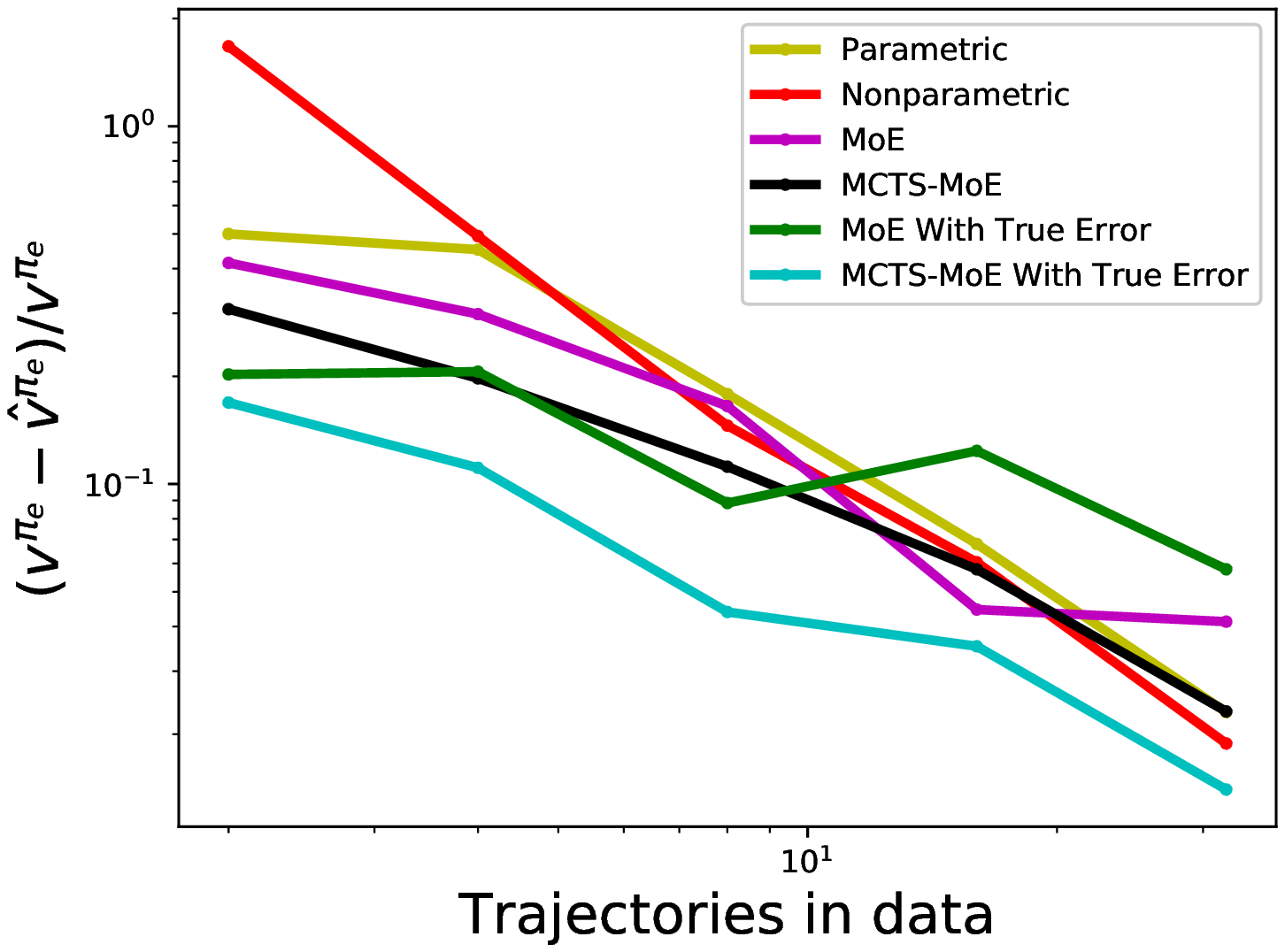}}
\subcaptionbox{HIV\label{fig:hiv_consistency_experiments}}
{\includegraphics[width=0.23\textwidth]{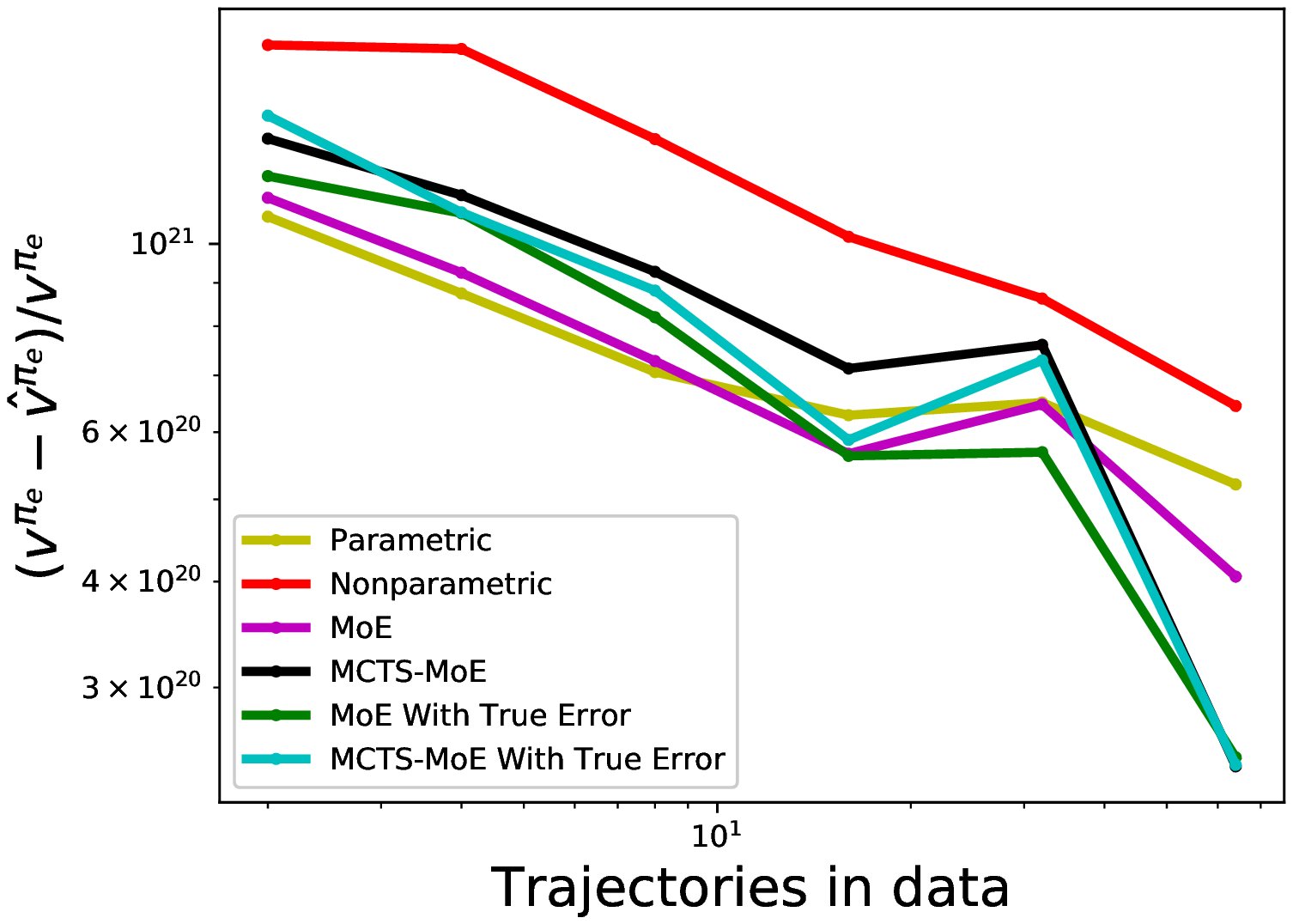}}
\caption{\textbf{Empirical check of consistency.} For both medical simulators the value estimation error decreases as the number of observed trajectories is increased. For both domains we the behavior policy is the $\epsilon$-greed policy with $\epsilon=0.4$.}
\label{fig:consistency_experiments}
\end{figure}

\subsection{Effect of the metric on value estimation for HIV}
\label{sec:different_metric_for_hiv}

When presenting the results for the HIV simulator, we noted that for high values of randomness in the behavior policy, the MCTS-MoE is outperformed by the parametric model and the greedy MoE, despite performing well in terms of the trajectory error. We argued this effect can be attributed to the distance metric used to quantify the transition error, which does not take into account the fact that some dimensions are more strongly correlated with the reward than others. This claim is further supported by the observation that in the regime where the MCTS-MoE performs poorly in terms of value estimation, the nonparametric performs significantly worse than all other methods, despite performing reasonably well in terms of trajectory error.

To further investigate the effect of the metric we ran our experiments again but used a metric which gives 20 times more weight to the $6^{th}$ dimension in the state space. This dimension (bottom right plot in Figures \ref{fig:hiv_optimal_policy_true_traj}) and \ref{fig:hiv_random_behavior_in_ss_behavior_traj}) represents the number of immune effectors in the patient's body and is most strongly correlated with the reward. In Figure \ref{fig:different_metric_for_hiv} we present the results for OPE on the HIV simulator with this new metric and demonstrate that indeed using this new metric improves the performance of the MCTS-MoE simulator in terms of value estimation, at the cost of degrading the trajectory error.

\begin{figure}[t]
\centering
\subcaptionbox{HIV\label{fig:hiv_traj_err_diff_metric}}
{\includegraphics[width=0.23\textwidth]{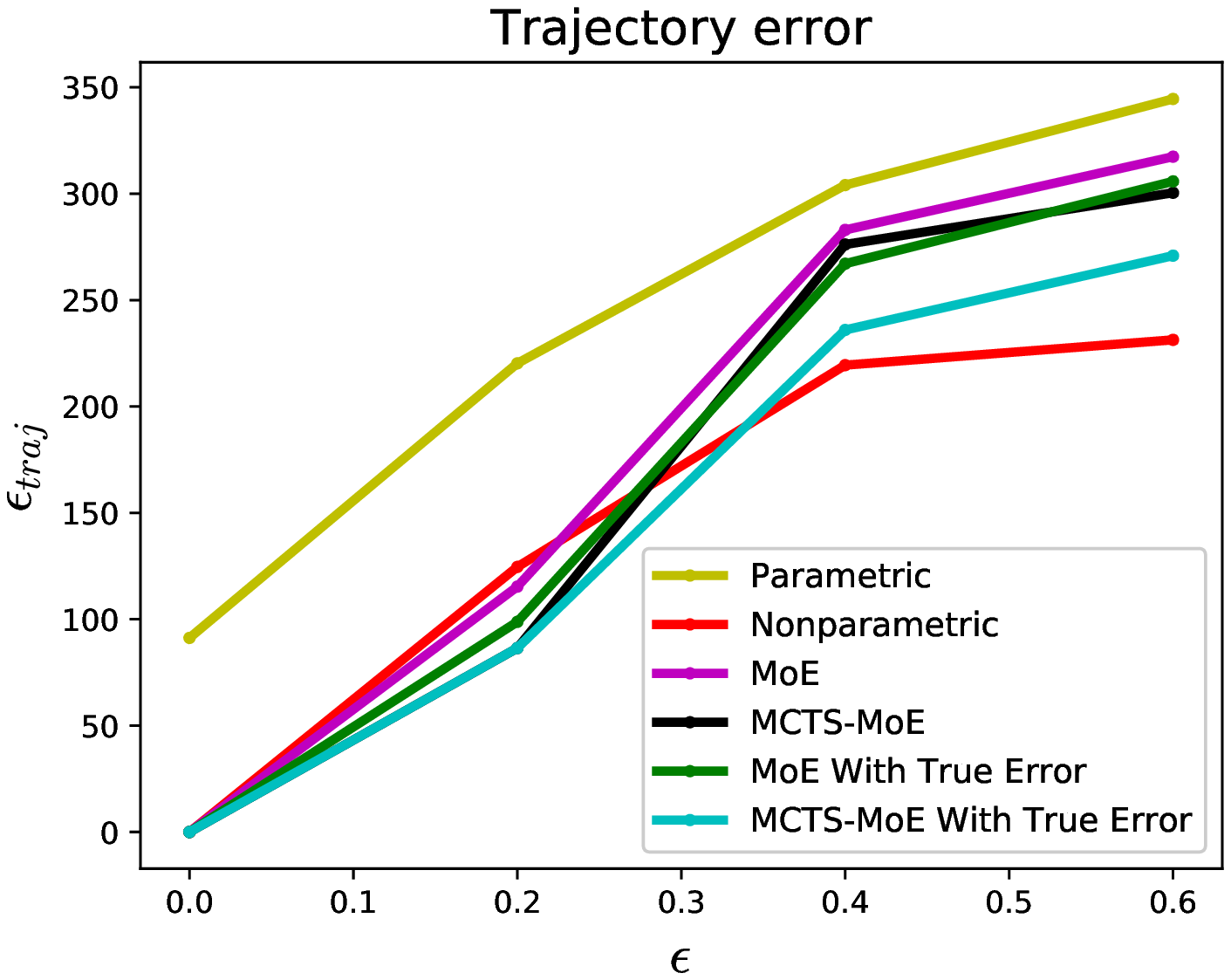}}
\subcaptionbox{HIV\label{fig:hiv_eval_policy_value_est_err_squared_diff_metric}}
{\includegraphics[width=0.23\textwidth]{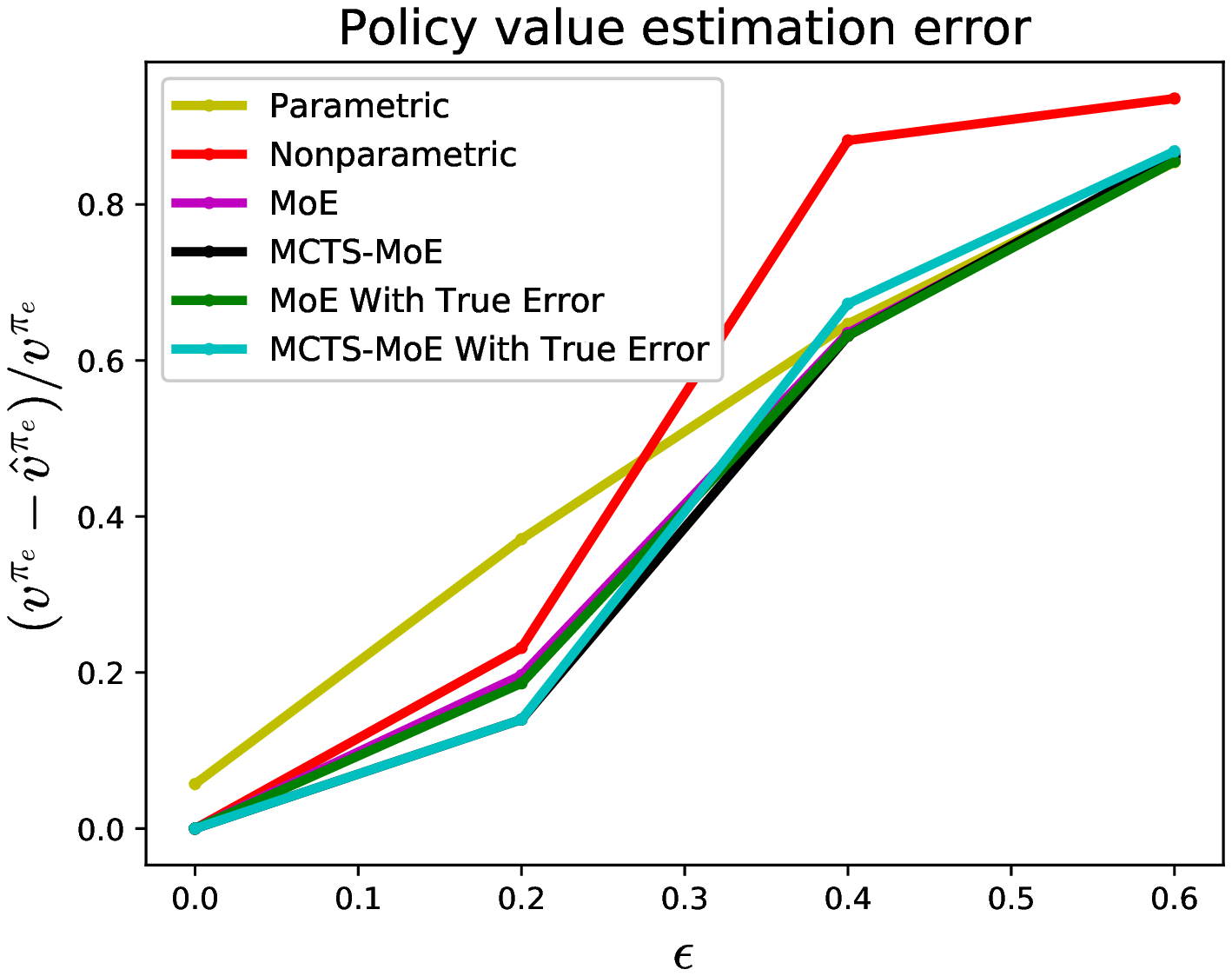}}
\caption{\textbf{Effect of the metric.} By replacing the Euclidean distance with a metric that puts more weight on state dimensions which are strongly correlated with reward, the value estimation performance of the MCTS-MoE can be improved at the cost of trajectory error.}
\label{fig:different_metric_for_hiv}
\end{figure}

\end{appendices}

\end{document}